\documentclass{article}

\usepackage{microtype}
\usepackage{graphicx}
\usepackage{subcaption}
\usepackage{booktabs}
\usepackage{hyperref}

\usepackage[accepted]{icml2026}

\usepackage{amsmath}
\usepackage{amssymb}
\usepackage{mathtools}
\usepackage{amsthm}
\usepackage{amsfonts}
\usepackage{bbm}
\usepackage{multirow}
\usepackage{adjustbox}
\usepackage{enumitem}
\usepackage{url}
\usepackage{xcolor}
\usepackage{bm}

\usepackage{thmtools}
\usepackage{thm-restate}

\usepackage{algpseudocode}

\usepackage[capitalize,noabbrev,nameinlink]{cleveref}

\newcommand{\E}{\mathbb{E}}

\newcommand{\cC}{\mathcal{C}}
\newcommand{\cZ}{\mathcal{Z}}

\newcommand{\bpi}{\boldsymbol{\pi}}

\DeclareMathOperator*{\scat}{scat}

\newcommand{\DD}{\mathbb{D}}
\newcommand{\fd}{\mathsf{d}}
\newcommand{\Wone}{\mathsf{W}_1}
\newcommand{\LCE}{\mathcal{L}_{\mathrm{CE}}}
\newcommand{\LCEbatch}{\widehat{\mathcal{L}}_{\mathrm{CE}}}
\newcommand{\LDA}{\mathcal{L}_{\mathrm{DA}}}
\newcommand{\LDAbatch}{\widehat{\mathcal{L}}_{\mathrm{DA}}}

\theoremstyle{plain}
\newtheorem{theorem}{Theorem}[section]

\theoremstyle{definition}

\newtheorem{assumption}[theorem]{Assumption}
\theoremstyle{remark}
\newtheorem{remark}[theorem]{Remark}

\usepackage[textsize=tiny]{todonotes}

\icmltitlerunning{Robust Domain Generalization under Divergent Marginal and Conditional Distributions}

\begin{document}

\icmlsetsymbol{equal}{*}
\icmlsetsymbol{corr}{$\dagger$}

\twocolumn[
  \icmltitle{Robust Domain Generalization\\ under Divergent Marginal and Conditional Distributions}

  \begin{icmlauthorlist}
    \icmlauthor{Jewon Yeom}{snu}
    \icmlauthor{Kyubyung Chae}{snu}
    \icmlauthor{Hyunggyu Lim}{snu}
    \icmlauthor{Yoonna Oh}{sk}
    \icmlauthor{Dongyoon Yang}{sk,corr}
    \icmlauthor{Taesup Kim}{snu,corr}
  \end{icmlauthorlist}

  \icmlaffiliation{snu}{Graduate School of Data Science, Seoul National University, South Korea}
  \icmlaffiliation{sk}{SK hynix, South Korea}

  \icmlcorrespondingauthor{Dongyoon Yang}{ydy0415@gmail.com}
  \icmlcorrespondingauthor{Taesup Kim}{taesup.kim@snu.ac.kr}

  \icmlkeywords{Domain Generalization, Meta-Learning, Long-Tailed Recognition, ICML}

  \vskip 0.3in
]

\printAffiliationsAndNotice{}

\begin{abstract}
Domain generalization (DG) aims to learn predictive models that can generalize to unseen domains.
Most existing DG approaches focus on learning domain-invariant representations under the assumption of conditional distribution shift (i.e., primarily addressing changes in $P(X\mid Y)$ while assuming $P(Y)$ remains stable). 
However, real-world scenarios with multiple domains often involve compound distribution shifts where both the marginal label distribution $P(Y)$ and the conditional distribution $P(X\mid Y)$ vary simultaneously.
To address this, we propose a unified framework for robust domain generalization under divergent marginal and conditional distributions. 
We derive a novel risk bound for unseen domains by explicitly decomposing the joint distribution into marginal and conditional components and characterizing risk gaps arising from both sources of divergence. 
To operationalize this bound, we design a meta-learning procedure that minimizes and validates the proposed risk bound across seen domains, ensuring strong generalization to unseen ones. 
Empirical evaluations demonstrate that our method achieves state-of-the-art performance not only on conventional DG benchmarks but also in challenging multi-domain long-tailed recognition settings where both marginal and conditional shifts are pronounced.
\end{abstract}

\section{Introduction}
Modern machine learning models are predominantly trained on static, homogeneous datasets where the joint distribution over inputs $X$ and labels $Y$ is assumed to remain consistent. 
However, real-world deployments rarely afford such stability.
Once models move beyond their curated training domains and begin operating in real-world environments, they must contend with shifts in geography, time, sensors, users, or even tasks. 
These changes often introduce compound distributional shifts that severely degrade performance.

Two major components of distributional shift commonly co-occur in practice: (i) \emph{shifts in the marginal label distribution} $P(Y)$, and (ii) \emph{shifts in the class-conditional distribution} $P(X \mid Y)$. 
The former reflects changes in class frequencies (e.g., varying species abundance across habitats), while the latter encompasses changes in the appearance or representation of each class (e.g., sketch vs. photo)

Despite their practical relevance, most prior work in domain generalization (DG) assumes that only the conditional distribution $P(X \mid Y)$ varies across domains, while $P(Y)$ remains fixed~\citep{DG_survey}.
This assumption underpins many popular DG methods that focus on learning domain-invariant representations. 
However, these approaches neglect the impact of marginal shifts, which can significantly affect model calibration and decision thresholds, especially in long-tailed or imbalanced settings \citep{GINET}. 
In such cases, models may become overconfident in dominant classes, under-represent rare classes, and fail to adjust to shifting class priors, leading to degraded performance in unseen domains \citep{GMDG}. On the other hand, approaches aimed at addressing label imbalance or long-tail distributions often operate within a single domain, failing to account for representational drift across environments \citep{MDLT_survey}. 
Thus, there remains a critical gap: 
\begin{quote}  
\vspace{-0.3cm}
\emph{How can we build models that generalize robustly under simultaneous marginal and conditional distribution shifts?}
\vspace{-0.3cm}
\end{quote}

In this work, we take a principled approach to tackle this challenge.
We begin by revisiting the upper bounds for domain generalization, explicitly decomposing the joint distribution shift into marginal and conditional components. 
Unlike prior works that rely on heuristic alignment objectives, our theoretical analysis explicitly decomposes the joint distribution shift into marginal and conditional components, deriving a rigorous risk bound tailored for unseen domains.
Building on this insight, we propose a novel meta-learning framework, \textsc{RC-Align}, that minimizes and validates this bound across a diverse set of source domains. 
Our method trains models to actively anticipate and counteract shifts in both $P(Y)$ and $P(X \mid Y)$, rather than passively hoping for invariant features.

To evaluate our approach, we conduct extensive experiments on standard domain generalization benchmarks as well as designed multi-domain long-tailed recognition scenarios where both sources of shift are pronounced. 
Our results demonstrate that the proposed method not only outperforms state-of-the-art DG methods but also exhibits remarkable robustness in the face of compounded distributional shifts.
In summary, this paper makes the following key contributions:
\vspace{-0.5cm}
\begin{itemize}[nosep, leftmargin=1em]
    \item We derive a rigorous risk bound for unseen domains by explicitly decomposing the joint distribution shift into marginal and conditional components, thereby mathematically identifying the specific sources of generalization failure under compound shifts.
    \item 
    Inspired by theoretical bound, we introduce \textsc{RC-Align}, a principled meta-learning framework that minimizes the derived bound by utilizing a Domain-Class Distribution Alignment loss proven to strictly upper-bound the feature mismatch term.
    \item Our extensive evaluations demonstrate that \textsc{RC-Align} exhibits superior robustness on both the standard DG and challenging Multi-Domain Long-Tailed Recognition (MDLT) benchmarks.
\end{itemize}

\section{Related Work}

\textbf{Domain Generalization under Divergent Distributions.}
Classical domain generalization (DG) aims to learn models that perform well on unseen domains by mitigating shifts in the class-conditional distribution $P(X \mid Y)$. Approaches include domain-invariant representation learning such as correlation alignment \citep{DeepCORAL}, adversarial feature alignment \citep{DANN}, and conditional invariance \citep{li2018conditional}, as well as optimization-based techniques \citep{SAGM, SWAD} and augmentation strategies \citep{Mixstyle, ManifoldMixup}. However, most DG methods implicitly assume balanced label distributions across domains.

In contrast, long-tailed (LT) recognition addresses shifts in the marginal label distribution $P(Y)$, where head classes dominate and tail classes are under-represented. Remedies include resampling \citep{systematic_resampling}, reweighting \citep{learningTail_reweight, Lin2017Focal}, logit adjustment \citep{Ren2020BalancedSoftmax}, and two-stage or ensemble strategies \citep{Kang2020cRT, contrastiveLT, CBDistill}. While effective in single-domain LT settings, these methods often degrade under domain shifts.

The compounded challenge defines \emph{multi-domain long-tailed recognition} (MDLT), where both $P(X \mid Y)$ and $P(Y)$ diverge across domains, making tail-class representations especially fragile \citep{gu2022ltds}. Recent works address this by defining domain--class transferability \citep{Boda}, modeling imbalanced domain adaptation \citep{DIDA}, or using a meta-learning framework for LT under domain shifts \citep{gu2022ltds}. A key insight is that label imbalance is an intrinsic but overlooked factor in DG: beyond MDLT, handling cross-domain imbalance can improve standard DG benchmarks and complement existing algorithms, suggesting it is largely orthogonal to prior DG remedies \citep{Boda}. Despite progress, a unified framework that jointly tackles conditional shifts and marginal imbalance remains open.

\textbf{Meta-Learning for DG.}
Meta-learning has been widely used for DG by simulating domain shifts through episodic training. Bilevel methods such as MLDG \citep{MLDG} and MetaReg \citep{MetaReg} improve robustness by splitting source domains or learning transferable regularizers, while later work emphasizes gradient-based objectives (update alignment~\citep{Fish}, gradient averaging~\citep{Arith}) and bilevel consistency for invariance \citep{MLIR}. Distinct from these, our method introduces a \emph{robust compound alignment} that integrates manifold mixup with episodic adaptation, jointly addressing cross-domain discrepancy and intra-domain imbalance beyond prior approaches.

\textbf{Risk Bounds and Theoretical Foundations under Distribution Shifts.}
An important line of research provides theoretical guarantees for learning under distribution shifts. \citet{SSP} investigate long-tailed regimes, demonstrating how semi- and self-supervised signals can alleviate bias and improve generalization under severe imbalance. Extending this perspective, \citet{Boda} formalize domain--class relations in multi-domain long-tailed recognition and derive a transferability bound highly correlated with test accuracy. More recently, \citet{TAROT} focus on robust domain adaptation, introducing a margin-based divergence and establishing generalization bounds on adversarially robust target risk. 

These theoretical insights also resonate with a recurring observation in domain generalization: minimizing empirical source risk alone does not necessarily control the target risk, even when target domains are formed as mixtures or reweightings of source domains~\cite{muandet2013domain,blanchard2021mtl_marginal_transfer_learning,krueger2020vrex}. This is largely because the risk and shift terms are coupled through the learned representation and can be influenced by source-specific shortcuts~\cite{IRM,izmailov2022feature,chen2023understanding,deng2023robust}.

Building upon these advances, our work develops a unified risk bound that explicitly decomposes both marginal and conditional divergences and connects the resulting terms to a practical training objective under compound shifts.

\section{Upper Bound for Domain Generalization under Compound Shifts}\label{sec:theory}

A central challenge in domain generalization (DG) is evaluating how well a model trained on a collection of source domains will perform on an unseen target domain. Since no data from the target domain is available during training, a natural and widely adopted framework is the leave-one-domain-out (LODO) setting, where one domain is held out as a proxy for the unseen target, and the model is trained on the remaining source domains. This episodic formulation not only mimics real-world deployment conditions but also provides a rigorous setting for analyzing domain shift sensitivity and generalization behavior. 

Formally, we consider $K$ source domains $\DD = \{\mathcal{D}_1, \dots, \mathcal{D}_K\}$. The model consists of a feature extractor $f_\theta$ and a classifier $g_\phi$, with combined parameters $\Theta = \{\theta, \phi\}$. In each LODO episode, we designate one domain $\mathcal{D}_i$ as the query domain (unseen during training) and the remaining domains $\mathcal{D}_{-i}$ as the support mixture. The goal is to understand and control the unseen-domain risk $\mathcal{R}_{\mathcal{D}_i}(\Theta)$, defined as:
\begin{equation*}
\mathcal{R}_{\mathcal{D}_i}(\Theta) = \mathbb{E}_{(x, y)\sim \mathcal{D}_i}\left[\ell\left(g_\phi\left(f_\theta\left(x\right)\right), y\right)\right],
\end{equation*}
where $\ell$ is a bounded, $L_\ell$-Lipschitz loss function with respect to the distance metric $\fd$ in the feature space $\cZ$ (See \Cref{app:C_setup} for details).

However, measuring the absolute value of $\mathcal{R}_{\mathcal{D}_i}(\Theta)$ provides limited insight into what causes generalization failures. To make this analysis actionable, we consider how and why the query (i.e., unseen, target) domain risk deviates from the risk over the support mixture (i.e., source) domain $\mathcal{R}_{\mathcal{D}_{-i}}(\Theta)$ (i.e. $\mathcal{R}_{-i}(\Theta) = \sum_{k \neq i} P(\mathcal{D}= \mathcal{D}_k)\mathcal{R}_{k} (\Theta)$). This motivates a decomposition of the generalization gap $\mathcal{R}_{\mathcal{D}_i}(\Theta) - \mathcal{R}_{\mathcal{D}_{-i}}(\Theta)$.

To achieve a tractable and interpretable analysis, we propose to decompose each domain’s joint distribution into two components: (1) the marginal label distribution (i.e., prior) $P(Y=c|\mathcal{D}_i) := \pi_{i}(c)$, and (2) the class-conditional feature distribution $P(Z=f_{\theta}(x)\mid Y=c, \mathcal{D}_i):=P_{i,c}$. 

This decomposition enables us to separate the sources of domain shift into two orthogonal axes: (1) Prior mismatch:  how the label frequencies differ across domains, and (2) Feature mismatch: how the conditional feature distributions shift for each class.

Such decomposition is crucial because existing DG methods often treat domain shift as a single nuisance variable (e.g., via domain-invariant representations), without disentangling what kind of shift is occurring. In contrast, our analysis shows that even if features are invariant, shifts in label priors alone can lead to significant performance degradation and vice versa. This leads to our first key theoretical result:

\begin{assumption}
\label{ass:main}
The loss function $\ell$ is bounded and $L_\ell$-Lipschitz continuous with respect to the distance metric $\fd$ in the feature space $\cZ$.
\end{assumption}

\begin{restatable}{theorem}{thmQuerySupportBound}[Query--support bound]\label{thm:qs}
Under \Cref{ass:main}, the risk on a query domain ($\mathcal{D}_i$) is bounded by the risk on the support mixture ($\mathcal{D}_{-i}$) plus mismatch terms:
\begin{equation}\label{eq:qs}
\begin{aligned}
\mathcal{R}_{\mathcal{D}_i}(\Theta) \le \mathcal{R}_{\mathcal{D}_{-i}}(\Theta) 
&+ \underbrace{\sum_{c \in \cC} | \pi_i(c) - \pi_{-i}(c) | \cdot \mathcal{R}_{-i,c}(\Theta)}_{\text{Prior Shift}} \\
&+ \underbrace{L_\ell \sum_{c\in\cC} \pi_i(c) \, \Wone(P_{i,c}, P_{-i,c})}_{\text{Feature Shift}},
\end{aligned}
\end{equation}
where $\Wone(\cdot,\cdot)$ and $L_\ell$ are the 1-Wasserstein distance and the Lipschitz constant of $\ell$, respectively.
\end{restatable}

All proofs of theoretical results in this section (\Cref{thm:qs,thm:qs-DA,thm:mix,lem:one-step-contraction,prop:DA-reduction}) are provided in \Cref{app:A,app:B}. 

\textbf{Implications of \Cref{thm:qs}.}
This theorem explicitly separates the challenges posed by compound shifts (i.e. prior and feature shifts). Crucially, the impact of the Prior Shift is modulated by the actual class-wise support risks, $\mathcal{R}_{-i,c}(\Theta)$, rather than a constant bound. Minimizing the overall support risk $\mathcal{R}_{\mathcal{D}_{-i}}(\Theta)$ (via Cross-Entropy) inherently minimizes $\mathcal{R}_{-i,c}(\Theta)$, thereby actively controlling the adverse effects of label distribution mismatch.

\subsection{Controlling The Query-Support Gap: \textit{A Tractable Bound}}\label{sec:control_qs_gap}

While \Cref{thm:qs} identifies the sources of the gap for domain generalization ($\mathcal{R}_{\mathcal{D}_i}(\Theta)\ - \mathcal{R}_{\mathcal{D}_{-i}}(\Theta)$), the Feature Shift term involves the intractable 1-Wasserstein distance ($\Wone$) between feature distributions. We now establish a connection between this term and the practical Domain-Class Distribution Alignment (DA) loss~\citep{Boda} ($\LDA$) used in our algorithm.

We utilize a series of analytical steps involving an InfoNCE decomposition of the DA loss and a two-step transport argument that routes mass through the source centroids. This allows us to bound the Wasserstein distance using the expected DA loss. (Detailed derivations involving \Cref{lem:nce}, \Cref{lem:phi}, and \Cref{lem:two-step} are provided in \Cref{app:C_lemmas}).

\begin{restatable}{theorem}{thmQsDaBound}[DA loss controls the query--support gap]\label{thm:qs-DA}
The total feature distribution mismatch (Feature Shift term) is upper-bounded by the DA loss:
\begin{equation}\label{eq:qs-DA}
\sum_{c}\pi_i(c)\,\Wone\!\big(P_{i,c},P_{-i,c}\big)
\ \le\
\E\,\LDA^{(i)}(\theta)\;+\;C_{\mathrm{scat}}^{(i)}\;+\;R,
\end{equation}
where $C_{\mathrm{scat}}^{(i)}$ is the average weighted scatter within the support domains (defined formally in \Cref{app:C_lemmas}), and $R$ is the diameter of the feature space.
\end{restatable}

\textbf{The Final Theoretical Bound.}
Combining \Cref{thm:qs} and \Cref{thm:qs-DA} yields our main theoretical result, connecting the unobserved target risk to concrete, optimizable quantities:
\begin{equation}\label{eq:qs-final}
\begin{aligned}
\mathcal{R}_{\mathcal{D}_i}(\Theta) \le \mathcal{R}_{\mathcal{D}_{-i}}(\Theta) &+ \sum_{c} | \Delta\pi_c | \mathcal{R}_{-i,c}(\Theta) \\
&+ L_\ell\Big(\E\,\LDA^{(i)}(\theta)+C_{\mathrm{scat}}^{(i)}+R\Big),
\end{aligned}
\end{equation}
where $\Delta\pi_c = \pi_i(c)-\pi_{-i}(c)$.

\textbf{Significance of the Bound.}
This bound (Eq.~\eqref{eq:qs-final}) provides the justification for our proposed methodology:
\begin{itemize}[leftmargin=1em]
\item \textbf{Controlling Feature Shift.} Our analysis proves that the practical {DA loss $\LDA$ directly upper-bounds the feature mismatch term.} This provides a stronger, risk-based justification than the statistical bounds in prior work \citep{Boda}.
\item \textbf{Controlling Prior Shift.} The impact of prior mismatch is {modulated by the model's class-wise source performance} $\mathcal{R}_{-i,c}(\Theta)$.
\item \textbf{Synergistic Optimization.} Our composite objective minimizes this bound through complementary mechanisms. $\LCE$ primarily reduces the source risks $\mathcal{R}_{-i,c}(\Theta)$. 
Crucially, 
optimizing the classification loss
inherently promotes intra-class compactness in the feature space, thereby \emph{implicitly controlling} the scatter term $C_{\mathrm{scat}}^{(i)}$.
\end{itemize}

\textbf{Insufficiency of Standard ERM.} 
This analysis explicitly highlights why standard Empirical Risk Minimization (ERM) is insufficient for domain generalization. ERM solely minimizes the source risk component ($\mathcal{R}_{\mathcal{D}_{-i}}$) in Eq.~\eqref{eq:qs-final}. However, the risk and distribution shift terms are coupled; minimizing source risk alone often leads to models learning spurious correlations (shortcuts), which paradoxically increases the feature shift term ($W_1$). Unlike ERM, our framework explicitly constrains this shift via $\LDA$, ensuring that performance gains on source domains genuinely transfer to unseen targets.

\subsection{Generalization to Arbitrary Unseen Target Domains}\label{sec:general_arb}

We generalize the analysis to an arbitrary unseen target domain $\mathcal{T}$ by relating it to a mixture of source domains, $\widetilde{\mathcal{T}}_{\boldsymbol{\pi}} := \sum_{i=1}^K \pi_i \mathcal{D}_i$, where $\boldsymbol{\pi} \in \Delta^K$ (the $K$-dimensional probability simplex).

\begin{restatable}{theorem}{thmMixBound}[Mixture-aware target bound]\label{thm:mix}
For any model $\Theta$ and $\bpi \in \Delta^K$, the risk on $\mathcal{T}$ is bounded by:
\begin{equation}\label{eq:mix}
\mathcal{R}_{\mathcal{T}}(\Theta)\ \le\ \sum_{i=1}^K \pi_i\,\mathcal{R}_{\mathcal{D}_i}(\Theta)\;+\;L_\ell\,\Wone\!\big(\mathcal{T},\widetilde{\mathcal{T}}_{\boldsymbol{\pi}}\big)\;.
\end{equation}
\end{restatable}

Combining \Cref{thm:mix} with the final theoretical bound (Eq. \eqref{eq:qs-final}) yields the comprehensive guarantee (detailed in \Cref{app:C_final_bound}), explicitly showing how controlling the terms leads to robust generalization on arbitrary target domains.

\subsection{Analysis of the Meta-Learning Procedure}\label{sec:maml_analysis}

We analyze the effectiveness of the meta-learning framework in minimizing the derived bound. We define the population inner objective $J_{-i}(\Theta)$ (combining risk and alignment) and the adaptation step $\Theta_i^+ := \Theta - \alpha \nabla_{\Theta} J_{-i}(\Theta)$.

To gain insights into the optimization dynamics of the inner loop, we analyze the convergence behavior under the standard assumptions of $L$-smoothness and the Polyak--Lojasiewicz (PL) condition (\Cref{ass:smooth-pl} in \Cref{app:C_maml}). While the PL condition is a strong assumption for the non-convex landscapes of deep networks, it serves as a standard analytical tool~\citep{karimi2016linear} to characterize local convergence behavior and understand the adaptation dynamics (See discussion in \Cref{app:C_maml}).

\begin{restatable}{lemma}{lemOneStepContraction}[One-Step Contraction]\label{lem:one-step-contraction}
Under standard smoothness/PL conditions, the inner adaptation step yields a geometric contraction of the objective:
\begin{equation}
\label{eq:contract-J}
J_{-i}(\Theta_i^+) - J_{-i}^\star \le \rho(\alpha) \big(J_{-i}(\Theta) - J_{-i}^\star\big),
\end{equation}
where the contraction factor $\rho(\alpha) < 1$.
\end{restatable}

This contraction suggests that the inner loop makes consistent progress within this idealized setting. We further analyze how this impacts the DA alignment term specifically.

\begin{restatable}{proposition}{propDaReduction}[Reduction of the DA alignment term]
\label{prop:DA-reduction}
The DA loss on the source domains after adaptation ($\Theta_i^+$) is bounded as:
\begin{equation}
\label{eq:DA-reduction}
\E\LDA^{(-i)}(\Theta_i^+)
\le \frac{\rho(\alpha)}{\lambda_{\mathrm{DA}}} J_{-i}(\Theta) + \frac{1-\rho(\alpha)}{\lambda_{\mathrm{DA}}} J_{-i}^\star.
\end{equation}
\end{restatable}

\Cref{prop:DA-reduction} confirms that the inner adaptation step actively reduces the alignment objective at a geometric rate. The MLDG~\citep{MLDG}-style outer loop ensures that this improved source alignment translates to better target generalization by minimizing the bound established in Eq.~\eqref{eq:qs-final}. A further discussion on optimization dynamics is provided in \Cref{app:C_maml}.
\section{Proposed Methodology: \textsc{RC-Align}}
\label{sec:methodology}

Based on the theoretical insights from \Cref{sec:theory}, we propose \textsc{RC-Align} (Robust Compound Alignment), a meta-learning framework designed to minimize the upper bound for domain generalization (Eq.~\eqref{eq:qs-final}) under simultaneous marginal and conditional distribution shifts.




\subsection{Controlling Prior and Feature Shifts}\label{sec:loss}

To operationalize the bound in Eq.~\eqref{eq:qs-final}, we use a
combined
objective function that addresses both \textit{the prior shift impact} and \textit{the feature shift} terms.

\textbf{Classification Loss (\textit{Controlling Prior Shift Impact}).}
\Cref{thm:qs} shows that the impact of prior shift is modulated by the class-wise support risks $\mathcal{R}_{-i,c}(\Theta)$. We use the standard empirical Cross-Entropy (CE) loss to minimize these risks. Given a batch $B$:
\begin{equation}
\label{eq:CE_batch}
\LCEbatch(B; \Theta) = \frac{1}{|B|} \sum_{(x,y) \in B} \mathrm{CE}(g_\phi(f_\theta(x)),y).
\end{equation}
By improving classification accuracy on the source domains, $\LCEbatch$ actively reduces the sensitivity of the model to changes in label priors.

\textbf{Domain-Class Distribution Alignment Loss (\textit{Controlling Feature Shift}).}
\Cref{thm:qs-DA} proves that the DA loss upper-bounds the feature mismatch term $W_1(P_{i,c}, P_{-i,c})$. We employ a contrastive loss to pull features closer to the centroids of the same class in other domains, thereby minimizing this mismatch. 

Given a batch $B$ and pre-computed support centroids $\{\mu\}$, the loss is defined as:
\begin{equation}
\label{eq:DA_batch}
\widehat{\mathcal{L}}_{\text{DA}}(B; \theta, \{\mu\}) = \frac{1}{|B|} \sum_{(x,y,i) \in B} \ell_{\text{align}}(x, y, i),
\end{equation}
where the sample-wise alignment loss $\ell_{\text{align}}(x, y, i)$ is given by:
\begin{equation}
\label{eq:DA_sample}
\begin{aligned}
\frac{-1}{K-1} \sum_{j \neq i} \log \frac{\exp(-\text{d}(f_\theta(x), \mu_{j, y}))}{\sum_{(j',c') \in \mathcal{S}_{i,y}} \exp(-\text{d}(f_\theta(x), \mu_{j',c'}))},
\end{aligned}
\end{equation}
Here, $\mathcal{S}_{i,y}$ denotes the set of all domain-class pairs in the support mixture excluding the query's identity.

%


\textbf{Total Training Objective.}
The objective used in the inner loop is the empirical weighted combination, directly targeting the optimizable terms of the upper bound for domain generalization:
\begin{equation}
\label{eq:total}
\resizebox{0.89\linewidth}{!}{$
\widehat{\mathcal{L}}_{\text{total}}(B; \Theta, \{\mu\}) =
\LCEbatch(B; \Theta) \;+\; \lambda_{\mathrm{DA}}\,
\LDAbatch(B; \theta, \{\mu\}),
$}
\end{equation}
where $\lambda_{\mathrm{DA}}$ is a regularization parameter that balances the two terms.

\subsection{The \textsc{RC-Align} with a Meta-learning Framework}\label{sec:meta}

Our proposed method, \textsc{RC-Align}, utilizes a two-stage, MAML~\citep{Finn2017MAML}-style procedure within the Leave-One-Domain-Out (LODO) protocol to ensure that the minimization of the composite objective translates to robust generalization on the unseen domain. The detailed training procedure is summarized in \Cref{alg:rc_align}.

\textbf{Inner Adaptation Loop.} 
In the inner loop, for a given held-out virtual target domain $\mathcal{D}_i$, the model's meta-parameters $\Theta$ are temporarily adapted using the remaining source domains $\mathcal{D}_{-i}$. This adaptation is guided by our composite objective (Eq.~\ref{eq:total}), which utilizes class centroids
$\{\mu_k\}_{k \neq i}$
estimated on-the-fly from current mini-batches for computational efficiency. 

While batch-based estimation can introduce noise, particularly for tail classes in the MDLT setting, the meta-learning process inherently learns to adapt robustly despite such perturbations. To further stabilize the optimization, we incorporate \textbf{Manifold Mixup}~\citep{ManifoldMixup}. Beyond its role as data augmentation, Mixup provides a theoretical advantage: by encouraging linearity in the feature space and smoothing the non-convex optimization landscape, it facilitates the \textit{geometric contraction} of the inner loop (as discussed in \Cref{sec:maml_analysis}). This adaptation step aims to minimize both source risk and feature misalignment simultaneously:
\begin{equation} \label{eq:inner_step}
\Theta' \leftarrow \Theta - \alpha \nabla_\Theta \widehat{\mathcal{L}}_{\text{total}}(\Theta; B_{-i}, \{\mu\}_{-i}).
\end{equation}
As proven in \Cref{prop:DA-reduction}, this step guarantees a formal reduction in the alignment objective.

\textbf{Outer Meta-Optimization Loop.} 
In the outer loop, we evaluate the performance of the adapted model $\Theta'$ on the previously held-out virtual target domain $\mathcal{D}_i$. The resulting meta-loss provides a gradient signal to update the original meta-parameters $\Theta$:
\begin{equation}
\Theta \leftarrow \Theta - \eta\,\nabla_\Theta \widehat{\mathcal{L}}_{\text{CE}}(B_i; \Theta').
\end{equation}
This meta-optimization process trains the model to find an initialization $\Theta$ that is primed for rapid and effective adaptation. By ensuring that the theoretical bound reduction achieved in the inner loop translates into improved performance on unseen target domains, the model achieves superior generalization. For computational efficiency, we employ the first-order MAML (FO-MAML) variant. For the full procedure, please refer to \Cref{app:algorithms}.

\begin{table*}[t!]
\centering
\caption{Summary of average domain generalization accuracy (\%) across four datasets.}
\label{tab:dg_avg_results}
\begin{adjustbox}{width=\textwidth}
\begin{tabular}{l|cccc|ccl|cccc|c}
\toprule
\multicolumn{6}{c}{\textbf{Baselines (Part 1)}} &\multicolumn{1}{c}{} & \multicolumn{6}{c}{\textbf{Baselines (Part 2)}} \\
\cmidrule(lr){1-6} \cmidrule(rl){8-13}
\textbf{Algorithm} & \textbf{PACS} & \textbf{VLCS} & \textbf{OfficeHome} & \textbf{TerraInc.} & \textbf{Avg.} & &\textbf{Algorithm} & \textbf{PACS} & \textbf{VLCS} & \textbf{OfficeHome} & \textbf{TerraInc.} & \textbf{Avg.}\\
\midrule\midrule
ERM & 85.5 & 77.5 & 66.5 & 46.1 & 68.9 & & ARM & 85.1 & 77.6 & 64.8 & 45.5 & 68.3 \\
IRM & 83.5 & 78.5 & 64.3 & 47.6 & 68.5 & & VREx & 84.9 & 78.3 & 66.4 & 46.4 & 69.0 \\
GroupDRO & 84.4 & 76.7 & 66.0 & 43.2 & 67.6 & & RSC & 85.2 & 77.1 & 65.5 & 46.6 & 68.6 \\
Mixup & 84.6 & 77.4 & 68.1 & 47.9 & 69.5 & & MetaReg & 83.6 & 76.7 & 67.6 & 48.2 & 69.0 \\
MLDG & 84.9 & 77.2 & 66.8 & 47.7 & 69.2 & & MLIR &86.8 & 80.7 & 69.8 & 51.0 & 72.1 \\
CORAL & 86.2 & 78.8 & 68.7 & 47.6 & 70.3 & & Mixstyle & 85.2 & 77.9 & 60.4 & 44.0 & 66.9 \\
MMD & 84.6 & 77.5 & 66.3 & 42.2 & 67.7 & & SWAD & 88.1 & 79.1 & 70.6 & 50.0 & 72.0 \\
DANN & 83.7 & 78.6 & 65.9 & 46.7 & 68.7 & & PCL & 88.7 & 78.0 & 71.6 & 52.1 & 72.6 \\
CDANN & 82.6 & 77.5 & 65.8 & 45.8 & 67.9 & & BoDA & 86.9 & 78.5 & 69.3 & 50.2 & 71.2 \\
MTL & 84.6 & 77.2 & 66.4 & 45.6 & 68.5 & & SAGM & 86.6 & 80.0 & 70.1 & 48.8 & 71.4 \\
SagNet & 86.3 & 77.8 & 68.1 & 48.6 & 70.2 & & iDAG & \textbf{88.8} & 76.9 & \textbf{71.8} & 46.1 & 70.9 \\
GMDG & 85.6 & 79.2 & 70.7 & 50.1 & 71.4 & & Arith & 86.5 & 79.4 & 69.4 & 48.1 & 70.9 \\
\midrule 
\textbf{Ours} & 87.5 & \textbf{81.0} & 70.9 & \textbf{52.6} & \textbf{73.0} & & \textbf{Ours} & 87.5 & \textbf{81.0} & 70.9 & \textbf{52.6} & \textbf{73.0} \\

\bottomrule
\end{tabular}
\end{adjustbox}
\end{table*}

\section{Experiments}

We evaluate \textsc{RC-Align} on two challenging settings: (1) standard Domain Generalization (DG) and (2) Multi-Domain Long-Tailed Recognition (MDLT).

\subsection{Experimental Setup}

\textbf{Datasets.}
We evaluate \textsc{RC-Align}across two distinct scenarios. For standard Domain Generalization (DG), we utilize four widely adopted benchmarks: PACS~\citep{PACS}, VLCS~\citep{VLCS}, OfficeHome~\citep{OfficeHome}, and TerraIncognita~\citep{TerraInc}. To assess performance under challenging compound shifts (where both $P(Y)$ and $P(X|Y)$ diverge), we employ five Multi-Domain Long-Tailed Recognition (MDLT) benchmarks. Following established protocols~\citep{Boda}, these include the long-tailed versions of the DG datasets (PACS-MLT, VLCS-MLT, OfficeHome-MLT, TerraIncognita-MLT) and the large-scale DomainNet-MLT~\citep{DomainNet}. These MDLT setups induce significant and divergent long-tail class distributions across domains.

\textbf{Metrics.}
We follow the standard Leave-One-Domain-Out (LODO) evaluation protocol. We report the average accuracy across all held-out target domains (Average). For the MDLT setting, where robustness is critical, we additionally report the accuracy on the worst-performing domain (Worst) to assess the model's resilience under severe compound shifts.

\textbf{Implementation Details.}
We implement \textsc{RC-Align} within the DomainBed framework~\citep{DomainBed} to ensure reproducibility and fair comparison. Following standard protocol, we utilize a ResNet-50~\citep{ResNet50} backbone pretrained on ImageNet~\citep{Imagenet}. Models are optimized using the Adam optimizer~\citep{kingma2015adam}. We follow the training-domain validation procedure for hyperparameter selection. All experiments were conducted using NVIDIA RTX 3090 GPUs. Comprehensive implementation details, including dataset statistics, training protocols, and full hyperparameter configurations, are provided in \Cref{app:implementation}.

\subsection{Standard Domain Generalization Results}

We compare our method against established DG algorithms. For more detailed results, refer to \Cref{app:dg_details}.

\textbf{Analysis.}
As shown in \Cref{tab:dg_avg_results}, \textsc{RC-Align} achieves state-of-the-art performance, securing the highest average accuracy across all four standard DG benchmarks. Notably, our method sets a new top score on VLCS, and TerraIncognita and demonstrates highly competitive results on PACS, and OfficeHome, consistently ranking among the top-performing methods. The substantial 3.8\% average accuracy improvement over MLDG, another meta-learning baseline, is particularly telling. This highlights the effectiveness of our core proposal: incorporating an explicit Domain-Class Distribution Alignment (DA) loss within the meta-adaptation phase. While MLDG learns a generalizable initialization, our approach ensures this initialization is also primed to align class-conditional feature distributions, leading to more robust generalization even when class priors are balanced.

\subsection{Multi-Domain Long-Tailed Recognition Results}

We evaluate the robustness of \textsc{RC-Align} under compound shifts, comparing against specialized LT/MDLT methods. In all MDLT experiments, baseline results are from \citet{Boda}. We defer detailed descriptions of each baseline to~\Cref{app:mdlt_baseline}, and detailed results in~\Cref{app:mdlt_details}.

\begin{table*}[h]
\centering
\caption{Consolidated Mean (Average) and Worst Domain Accuracy across MDLT Benchmarks.}
\label{tab:consolidated_avg_worst}
\begin{adjustbox}{max width=\linewidth}
\begin{tabular}{l|cc|cc|cc|cc|cc}
\toprule
\multicolumn{1}{c}{}& \multicolumn{2}{c}{\textbf{VLCS-MLT}} & \multicolumn{2}{c}{\textbf{PACS-MLT}} & \multicolumn{2}{c}{\textbf{OfficeHome-MLT}} & \multicolumn{2}{c}{\textbf{TerraInc-MLT}} & \multicolumn{2}{c}{\textbf{DomainNet-MLT}} \\
\cmidrule(lr){2-3}\cmidrule(lr){4-5}\cmidrule(lr){6-7}\cmidrule(lr){8-9} \cmidrule(lr){10-11} 
\multicolumn{1}{c}{\textbf{Algorithm}}& \multicolumn{1}{c}{\textbf{Average}} & \multicolumn{1}{c}{\textbf{Worst}} & \multicolumn{1}{c}{\textbf{Average}} & \multicolumn{1}{c}{\textbf{Worst}} & \multicolumn{1}{c}{\textbf{Average}} & \multicolumn{1}{c}{\textbf{Worst}} & \multicolumn{1}{c}{\textbf{Average}} & \multicolumn{1}{c}{\textbf{Worst}} & \multicolumn{1}{c}{\textbf{Average}} & \multicolumn{1}{c}{\textbf{Worst}}\\
\midrule\midrule
ERM & 76.3 $\pm$0.4 & 53.6 $\pm$1.1 & 97.1 $\pm$0.1 & 95.8 $\pm$0.2 & 80.7 $\pm$0.0 & 71.3 $\pm$0.1 & 75.3 $\pm$0.3 & 67.4 $\pm$0.3 & 58.6 $\pm$0.2 & 29.4 $\pm$0.3 \\
IRM & 76.5 $\pm$0.2 & 52.3 $\pm$0.7 & 96.7 $\pm$0.2 & 95.2 $\pm$0.4 & 80.6 $\pm$0.4 & 70.7 $\pm$0.2 & 73.3 $\pm$0.7 & 64.3 $\pm$1.3 & 57.1 $\pm$0.1 & 27.6 $\pm$0.1 \\
GroupDRO & 76.7 $\pm$0.4 & 54.1 $\pm$1.3 & 97.0 $\pm$0.1 & 95.3 $\pm$0.4 & 80.1 $\pm$0.3 & 68.7 $\pm$0.9 & 72.0 $\pm$0.4 & 66.6 $\pm$0.2 & 53.6 $\pm$0.1 & 25.9 $\pm$0.3 \\
Mixup & 75.9 $\pm$0.1 & 52.7 $\pm$1.3 & 96.7 $\pm$0.2 & 95.1 $\pm$0.2 & 81.2 $\pm$0.2 & 72.3 $\pm$0.6 & 71.1 $\pm$0.7 & 60.4 $\pm$1.1 & 57.6 $\pm$0.1 & 28.7 $\pm$1.0 \\
MLDG & 76.9 $\pm$0.2 & 53.6 $\pm$0.5 & 96.6 $\pm$0.1 & 94.1 $\pm$0.3 & 80.4 $\pm$0.2 & 70.2 $\pm$0.6 & 76.6 $\pm$0.2 & 66.9 $\pm$0.5 & 58.5 $\pm$0.0 & 28.7 $\pm$0.0 \\
CORAL & 75.9 $\pm$0.5 & 51.6 $\pm$0.7 & 96.6 $\pm$0.5 & 94.3 $\pm$0.7 & 81.9 $\pm$0.1 & 72.7 $\pm$0.6 & 76.4 $\pm$0.5 & 67.8 $\pm$0.9 & 59.4 $\pm$0.1 & 30.1 $\pm$0.4 \\
MMD & 76.3 $\pm$0.6 & 53.4 $\pm$0.3 & 96.9 $\pm$0.1 & 96.2 $\pm$0.2 & 78.4 $\pm$0.4 & 70.7 $\pm$0.8 & 73.3 $\pm$0.4 & 63.7 $\pm$1.1 & 56.7 $\pm$0.0 & 27.2 $\pm$0.2 \\
DANN & 77.5 $\pm$0.1 & 54.1 $\pm$0.3 & 96.5 $\pm$0.0 & 94.3 $\pm$0.1 & 79.2 $\pm$0.2 & 67.7 $\pm$0.9 & 68.7 $\pm$0.9 & 61.1 $\pm$1.0 & 55.8 $\pm$0.1 & 26.9 $\pm$0.4 \\
CDANN & 76.6 $\pm$0.4 & 53.6 $\pm$0.4 & 96.1 $\pm$0.1 & 94.5 $\pm$0.2 & 79.0 $\pm$0.2 & 69.4 $\pm$0.3 & 70.3 $\pm$0.5 & 65.9 $\pm$1.0 & 56.0 $\pm$0.1 & 27.7 $\pm$0.4 \\
MTL & 76.3 $\pm$0.3 & 52.9 $\pm$0.5 & 96.7 $\pm$0.2 & 94.5 $\pm$0.6 & 79.5 $\pm$0.2 & 69.8 $\pm$1.6 & 75.0 $\pm$0.7 & 67.7 $\pm$1.4 & 58.6 $\pm$0.1 & 29.3 $\pm$0.2 \\
SagNet & 76.3 $\pm$0.2 & 52.3 $\pm$0.2 & 97.2 $\pm$0.1 & 95.2 $\pm$0.3 & 80.9 $\pm$0.1 & 70.5 $\pm$0.5 & 75.1 $\pm$1.6 & 66.5 $\pm$2.1 & 58.9 $\pm$0.0 & 29.4 $\pm$0.2 \\
Fish & 77.5 $\pm$0.3 & 54.3 $\pm$0.4 & 96.9 $\pm$0.2 & 95.2 $\pm$0.2 & 81.3 $\pm$0.3 & 71.3 $\pm$0.5 & 75.3 $\pm$0.5 & 66.3 $\pm$0.5 & 59.6 $\pm$0.1 & 29.1 $\pm$0.1 \\
Focal & 75.6 $\pm$0.4 & 52.3 $\pm$0.2 & 96.5 $\pm$0.2 & 94.6 $\pm$0.7 & 77.9 $\pm$0.0 & 67.6 $\pm$0.4 & 75.7 $\pm$0.4 & 65.3 $\pm$1.1 & 57.8 $\pm$0.2 & 27.5 $\pm$0.1 \\
CBLoss & 76.8 $\pm$0.3 & 52.5 $\pm$0.5 & 96.9 $\pm$0.1 & 95.1 $\pm$0.4 & 79.8 $\pm$0.2 & 69.5 $\pm$0.7 & 78.0 $\pm$0.4 & 68.3 $\pm$2.0 & 58.9 $\pm$0.1 & 30.1 $\pm$0.3 \\
LDAM & 77.5 $\pm$0.1 & 52.9 $\pm$0.2 & 96.5 $\pm$0.2 & 94.7 $\pm$0.2 & 80.3 $\pm$0.2 & 69.9 $\pm$0.5 & 74.7 $\pm$0.9 & 64.1 $\pm$1.4 & 59.2 $\pm$0.0 & 29.9 $\pm$0.2 \\
BSoftmax & 76.7 $\pm$0.5 & 52.9 $\pm$0.9 & 96.9 $\pm$0.3 & 95.6 $\pm$0.3 & 80.4 $\pm$0.2 & 70.9 $\pm$0.5 & 76.7 $\pm$1.0 & 65.6 $\pm$1.3 & 58.9 $\pm$0.1 & 29.9 $\pm$0.1 \\
SSP & 76.1 $\pm$0.3 & 52.3 $\pm$1.0 & 96.9 $\pm$0.2 & 95.4 $\pm$0.4 & 81.1 $\pm$0.3 & 71.1 $\pm$0.3 & 78.5 $\pm$0.7 & 67.3 $\pm$0.4 & 59.7 $\pm$0.0 & 31.6 $\pm$0.2 \\
CRT & 76.3 $\pm$0.2 & 51.4 $\pm$0.3 & 96.3 $\pm$0.1 & 94.9 $\pm$0.1 & 81.2 $\pm$0.0 & 72.5 $\pm$0.2 & 81.6 $\pm$0.1 & 70.0 $\pm$0.4 & 60.4 $\pm$0.2 & 32.6 $\pm$0.1 \\
BoDA$_r$ & 76.9 $\pm$0.5 & 51.4 $\pm$0.3 & 97.0 $\pm$0.1 & 95.1 $\pm$0.4 & 81.5 $\pm$0.1 & 71.8 $\pm$0.1 & 78.6 $\pm$0.4 & 68.5 $\pm$0.3 & 60.1 $\pm$0.2 & 32.2 $\pm$0.1 \\
BoDA-M$_r$ & 77.5 $\pm$0.3 & 53.4 $\pm$0.3 & 97.1 $\pm$0.1 & 94.9 $\pm$0.1 & 81.9 $\pm$0.2 & 71.6 $\pm$0.2 & 79.4 $\pm$0.6 & 71.3 $\pm$0.4 & 60.1 $\pm$0.2 & 32.2 $\pm$0.2 \\
BoDA$_{r,c}$ & 77.3 $\pm$0.2 & 53.4 $\pm$0.5 & 97.2 $\pm$0.1 & 95.7 $\pm$0.3 & 82.3 $\pm$0.1 & 72.3 $\pm$0.2 & 82.3 $\pm$0.3 & 68.5 $\pm$0.6 & 61.7 $\pm$0.1 & 33.4 $\pm$0.1 \\
BoDA-M$_{r,c}$ & 78.2 $\pm$0.4 & 55.4 $\pm$0.5 & 97.1 $\pm$0.2 & 96.3 $\pm$0.1 & 82.4 $\pm$0.2 & 72.3 $\pm$0.3 & 83.0 $\pm$0.4 & \textbf{74.6} $\pm$0.7 & 61.7 $\pm$0.2 & 33.3 $\pm$0.1 \\
\midrule
\textbf{Ours} & \textbf{79.1} $\pm$0.2 & \textbf{57.4} $\pm$0.8 & \textbf{98.4} $\pm$0.3 & \textbf{97.4} $\pm$0.5 & \textbf{83.3} $\pm$0.0 & \textbf{75.4} $\pm$0.1 & \textbf{83.9} $\pm$0.1 & 73.5 $\pm$0.4 & \textbf{63.8} $\pm$0.1 & \textbf{33.8} $\pm$0.1 \\
\bottomrule
\end{tabular}
\end{adjustbox}
\end{table*}


\textbf{Analysis.}
\Cref{tab:consolidated_avg_worst} demonstrates the effectiveness of \textsc{RC-Align} in the challenging MDLT setting. Our method consistently outperforms specialized LT approaches. Specifically, \textsc{RC-Align} sets a new state-of-the-art on all benchmarks for average, and 4 out of 5 for worst-domain accuracy. The particularly strong performance in the \textbf{Worst} metric highlights the robustness of our framework; for instance, it achieves a notable gain of 2.9\% on OfficeHome-MLT over the next best method and demonstrates superior robustness on the large-scale DomainNet-MLT as well. This empirical success confirms that \textsc{RC-Align} successfully learns representations that are robust not only to feature shifts ($P(X|Y)$) but also to severe and varying class imbalances (prior shifts). This validates our theoretical motivation: by directly optimizing a bound that accounts for both sources of divergence, \textsc{RC-Align} strikes a more effective balance than methods that primarily focus on either feature invariance or class re-balancing alone.

\subsection{Analysis and Validation}

We conduct further analysis to validate the components of our framework and confirm our theoretical findings.

\paragraph{Ablation Study.}

We analyze the contribution of each component within our framework on the VLCS dataset. For brevity in our reporting, we denote the Model-Agnostic Meta-Learning (MAML) framework as \textbf{M}, the Domain-Class Distribution Alignment loss as \textbf{DA}, and Manifold Mixup as \textbf{Mix}.
We evaluate performance using two key metrics: Average Accuracy (\textbf{Avg.}) and Worst-domain Accuracy (\textbf{Worst}).

\begin{table}[t]
\centering
\caption{Ablation study of RC-Align components on VLCS. Accuracies are reported as percentages (\%).}
\label{tab:ablation_vlcs}
\setlength{\tabcolsep}{5pt} 
\small 
\begin{tabular}{l|ccc|cc}
\toprule
\textbf{Method} & \textbf{M} & \textbf{DA} & \textbf{Mix} & \textbf{Avg.} & \textbf{Worst} \\
\midrule\midrule
ERM (Baseline)   &            &            &            & 77.5 & 64.3 \\
MLDG             & \checkmark &            &            & 77.2 & 65.2 \\
ERM + DA         &            & \checkmark &            & 79.6 & 66.9 \\
RC-Align (w/o Mix) & \checkmark & \checkmark &            & 80.2 & 67.4 \\
\midrule
\textbf{RC-Align (Full)} & \checkmark & \checkmark & \checkmark & \textbf{81.0} & \textbf{68.5} \\
\bottomrule
\end{tabular}
\end{table}

\begin{figure*}[t]
    \centering
    \includegraphics[width=\linewidth]{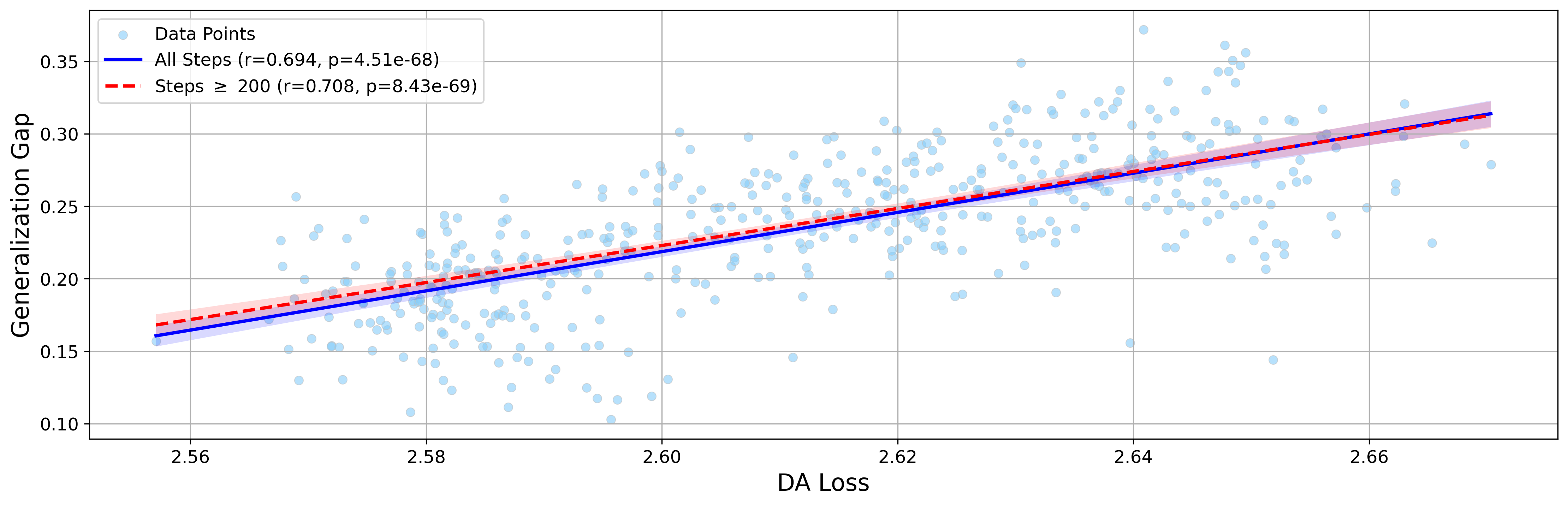}
    \caption{Empirical validation of Theorem~\ref{thm:qs-DA} on VLCS. The scatter plot of DA Loss vs. Generalization Gap reveals a strong positive correlation across all steps (blue line; $r=0.694, p<10^{-67}$), which strengthens during the stabilized phase (steps $\geq 200$, red dashed line; $r=0.708, p<10^{-68}$).}
    \label{fig:theory_validation_vlcs}
\end{figure*}

The results demonstrate the individual and synergistic contributions of each component within the \textsc{RC-Align} framework. As shown in Table \ref{tab:ablation_vlcs}, both the DA loss and the MAML framework (\textbf{M}) independently improve performance over the ERM baseline, particularly in the \textbf{Worst} metric, highlighting their effectiveness in mitigating feature and domain shifts. 

Combining these two elements in \textsc{RC-Align} (w/o Mix) yields a significant performance boost to 80.2\% Average accuracy, confirming their complementary roles in minimizing the theoretical generalization gap. Finally, the integration of Manifold Mixup (\textbf{Mix}) provides essential regularization that further stabilizes the adaptation process, leading to our best overall performance. This validates that the comprehensive design of our full \textsc{RC-Align} model is necessary for robust domain generalization.

\paragraph{Verification of Theoretical Findings.}
Our theoretical analysis (Theorem~\ref{thm:qs-DA}) posits that the Domain-Class Distribution Alignment (DA) loss upper-bounds the feature mismatch term, implying that minimizing the DA loss should reduce the generalization gap. To empirically validate this, we performed experiments on the VLCS dataset using a Leave-One-Domain-Out (LODO) setup. For each scenario, we designated one domain as the target and the rest as source domains, then recorded the target's DA loss and the generalization gap (difference between target error and average source error) every 10 training steps.

\cref{fig:theory_validation_vlcs}
empirically validates our theory, revealing a strong positive correlation between the DA loss and the generalization gap. This relationship notably strengthens during the stabilized training phase (steps $\geq 200$), suggesting that the DA loss becomes an even more precise indicator as features mature. This provides compelling evidence that the DA loss serves as a robust proxy for the feature mismatch term in our generalization bound (Eq.~\eqref{eq:qs-final}).

\begin{table}[t]
\centering
\caption{Computational cost comparison on OfficeHome. \textsc{RC-Align} achieves superior performance with a manageable increase in training time.}
\label{tab:computation_cost}
\small
\begin{tabular}{l|ccc}
\toprule
\bf Method & \bf Relative Time & \bf Max Memory & \bf Avg. Acc. \\
\midrule\midrule
ERM & 1.00$\times$ & 1.91 GB & 68.9 \\
MLDG & 1.16$\times$ & 1.88 GB & 69.2 \\
\midrule
\textbf{\textsc{RC-Align}} & \textbf{1.45$\times$} & \textbf{2.31 GB} & \textbf{73.0} \\
\bottomrule
\end{tabular}
\end{table}

\paragraph{Computational Efficiency.}
\label{sec:comp_efficiency}
To evaluate potential computational overhead, we compare the runtime and memory usage of \textsc{RC-Align} against ERM and MLDG on the OfficeHome dataset.
As shown in \Cref{tab:computation_cost}, \textsc{RC-Align} requires $1.45\times$ the training time of ERM. By utilizing FO-MAML to avoid expensive second-order derivatives, it achieves a $+4.1\%$ performance gain—a highly favorable trade-off for the improved robustness.

\section{Conclusion}

In this work, we addressed the critical challenge of domain generalization under compound distribution shifts, where both marginal label distributions ($P(Y)$) and conditional feature distributions ($P(X|Y)$) diverge across domains. We introduced \textsc{RC-Align}, a novel meta-learning framework designed to explicitly counteract these dual challenges. The principled design of \textsc{RC-Align} is grounded in a new upper bound for domain generalization that decomposes the risk from unseen domains into distinct terms for \textit{prior shift} and \textit{feature shift}. 

Our key theoretical contribution lies in proving that the practical Domain-Class Distribution Alignment (DA) loss, a core component of our method, serves as an \textit{upper bound on the 1-Wasserstein distance}, which quantifies feature mismatch in our bound.
This establishes a principled relationship between our training objective and the theoretical generalization gap. Furthermore, our analysis suggests that the meta-learning procedure helps minimize this bound by guiding source-domain alignment in a manner that improves performance on the target domain.

This theoretical foundation translates to strong empirical performance. \textsc{RC-Align} achieves state-of-the-art results not only on standard DG benchmarks but, more importantly, demonstrates superior robustness on challenging \textit{multi-domain long-tailed recognition} (MDLT) settings where compound shifts are most severe. Our analyses further validate this success, empirically confirming the correlation between the DA loss and the generalization gap predicted by our theory. By tightly integrating a practical meta-learning algorithm with a rigorous theoretical justification, \textsc{RC-Align} provides a principled approach for developing models that are robust to the complex, compound distribution shifts encountered in real-world deployments.

\section*{Impact Statement}
This paper presents a principled framework, \textsc{RC-Align}, to enhance the robustness of machine learning models under simultaneous shifts in label distributions $P(Y)$ and class-conditional distributions $P(X|Y)$. The societal implications of this work are two-fold:
First, it significantly improves the reliability of AI systems in high-stakes deployment scenarios, such as autonomous driving and medical imaging, where environmental conditions and class frequencies vary across locations. 
Second, by addressing the Multi-Domain Long-Tailed (MDLT) challenge, our method promotes algorithmic fairness. It ensures that the model performs robustly even for under-represented groups or rare events (the "tail" of the distribution), which often constitute the most critical cases in clinical or legal decision-making. 
While this work advances the field towards more trustworthy AI, users should remain cautious of residual errors in extreme out-of-distribution scenarios and employ these models as robust decision-support tools rather than autonomous authorities.

\bibliography{example_paper}
\bibliographystyle{icml2026}

\clearpage
\appendix
\onecolumn
\section{Detailed Proofs for upper bound for domain generalization}\label{app:A}

This appendix contains the detailed proofs for the theoretical results presented in Section \ref{sec:theory}, related to the upper bound for domain generalization and the DA loss analysis. We use the notation $\Theta = (\theta, \phi)$, $\fd$ for the distance metric, and $\mathcal{D}_i$ for domains.

\subsection{Proof of Theorem \ref{thm:qs}}\label{app:proof_thm_qs}

\thmQuerySupportBound*

\begin{proof}
Our goal is to bound the target risk $\mathcal{R}_{\mathcal{D}_i}(\Theta)$ using quantities related to the source mixture $\mathcal{R}_{\mathcal{D}_{-i}}(\Theta)$.

\textbf{Step 1: Decompose the Risk.}
First, we express the total risk as the sum of per-class expected losses.
We use the notation $\mathcal{R}_{\mathcal{D},c}(\Theta) = \E_{\mathbf{Z} \sim P_{\mathcal{D},c}}[\ell(g_\phi(\mathbf{Z}),c)]$.
\begin{align*}
\mathcal{R}_{\mathcal{D}_i}(\Theta) &= \sum_{c \in \cC} \pi_i(c) \, \mathcal{R}_{i,c}(\Theta) \\
\mathcal{R}_{\mathcal{D}_{-i}}(\Theta) &= \sum_{c \in \cC} \pi_{-i}(c) \, \mathcal{R}_{-i,c}(\Theta)
\end{align*}

\textbf{Step 2: Isolate Distributional Shifts.}
We add and subtract a hybrid term, $\sum_c \pi_i(c) \mathcal{R}_{-i,c}(\Theta)$, to separate the prior shift from the feature shift.
\begin{align*}
& \mathcal{R}_{\mathcal{D}_i}(\Theta) - \mathcal{R}_{\mathcal{D}_{-i}}(\Theta) \\
&= \sum_{c \in \cC} \left( \pi_i(c) \mathcal{R}_{i,c}(\Theta) - \pi_{-i}(c) \mathcal{R}_{-i,c}(\Theta) \right) \\
&= \underbrace{\sum_{c \in \cC} \pi_i(c) \left( \mathcal{R}_{i,c}(\Theta) - \mathcal{R}_{-i,c}(\Theta) \right)}_{\text{Term A: Feature Shift}} + \underbrace{\sum_{c \in \cC} \left( \pi_i(c) - \pi_{-i}(c) \right) \mathcal{R}_{-i,c}(\Theta)}_{\text{Term B: Prior Shift}}
\end{align*}

\textbf{Step 3: Analyze the Prior Shift Term (Term B).}
We bound Term B using the triangle inequality.
\begin{align*}
|\text{Term B}| &\le \sum_{c \in \cC} | \pi_i(c) - \pi_{-i}(c) | \cdot \mathcal{R}_{-i,c}(\Theta)
\end{align*}

\textbf{Step 4: Bound the Feature Shift Term (Term A).}
To bound this term, we use the Kantorovich-Rubinstein Duality theorem. Since $\ell(g_\phi(z),c)$ is $L_\ell$-Lipschitz with respect to $\fd$, we can apply a scaled version of the theorem:
$$
|\mathcal{R}_{i,c}(\Theta) - \mathcal{R}_{-i,c}(\Theta)| \le L_\ell \cdot \Wone(P_{i,c}, P_{-i,c})
$$
Now we can bound Term A:
\begin{align*}
|\text{Term A}| &\le \sum_{c \in \cC} \pi_i(c) \cdot L_\ell \cdot \Wone(P_{i,c}, P_{-i,c})
\end{align*}

\textbf{Step 5: Combine the Bounds.}
Combining the bounds for Term A and Term B using the triangle inequality yields the final result.
\end{proof}

\subsection{Proof of Theorem \ref{thm:qs-DA}}\label{app:proof_thm_qs-DA}

\thmQsDaBound*

\begin{proof}
This proof combines the results from the three lemmas detailed in Appendix \ref{app:C_lemmas}.

\textbf{Step 1: Average Over Classes.}
We take the expectation of the inequality from Lemma \ref{lem:two-step} over all classes $c$ according to $\pi_i(c)$.
\begin{align*}
\sum_{c} \pi_i(c) \Wone(P_{i,c}, P_{-i,c}) &\le \underbrace{\sum_{c} \pi_i(c) \sum_{j\neq i}\omega_j\,\E_{\mathbf{Z}\sim P_{i,c}} \fd(\mathbf{Z},\mu_{j,c})}_{\text{Term 1}} \\
&+ \underbrace{\sum_{c} \pi_i(c) \sum_{j\neq i} \omega_j S_{j,c}}_{\text{Term 2}}.
\end{align*}

\textbf{Step 2: Relate Terms to Definitions.}
\begin{itemize}
\item \textbf{Term 1} is the definition of the positive transport cost $A_i(\theta)$ (Lemma \ref{lem:nce}).
\item \textbf{Term 2} is the definition of $C_{\mathrm{scat}}^{(i)}$.
\end{itemize}
The inequality simplifies to:
\[
\sum_{c} \pi_i(c) \Wone(P_{i,c}, P_{-i,c}) \le A_i(\theta) + C_{\mathrm{scat}}^{(i)}.
\]

\textbf{Step 3: Substitute the Bound on $A_i$.}
We use the result from Lemma \ref{lem:phi}: $A_i(\theta) \le \E\,\LDA^{(i)}(\theta) + R$.
Substituting this gives the final result:
\[
\sum_{c} \pi_i(c) \Wone(P_{i,c}, P_{-i,c}) \le \E\,\LDA^{(i)}(\theta) + C_{\mathrm{scat}}^{(i)} + R.
\]
\end{proof}

\subsection{Proof of Theorem \ref{thm:mix}}\label{app:proof_thm_mix}

\thmMixBound*

\begin{proof}
The proof relies on the Lipschitz continuity of the risk functional and the linearity of expectation over mixtures.

\textbf{Step 1: Risk as a Lipschitz Functional.}
The expected risk $\mathcal{R}_{\mathcal{D}}(\Theta)$ is a functional of the distribution $\mathcal{D}$. Because the loss $\ell$ is $L_\ell$-Lipschitz w.r.t. $\fd$, the risk functional is $L_\ell$-Lipschitz w.r.t. $\Wone$. By the Kantorovich-Rubinstein duality:
\[
|\mathcal{R}_{\mathcal{T}}(\Theta) - \mathcal{R}_{\widetilde{\mathcal{T}}_{\boldsymbol{\pi}}}(\Theta)| \le L_\ell\,\Wone(\mathcal{T}, \widetilde{\mathcal{T}}_{\boldsymbol{\pi}}).
\]
Rearranging to isolate the target risk:
\[
\mathcal{R}_{\mathcal{T}}(\Theta) \le \mathcal{R}_{\widetilde{\mathcal{T}}_{\boldsymbol{\pi}}}(\Theta) + L_\ell\,\Wone(\mathcal{T}, \widetilde{\mathcal{T}}_{\boldsymbol{\pi}}).
\]

\textbf{Step 2: Linearity of Risk over Mixtures.}
By definition, the expectation over a mixture distribution is the weighted average of the expectations over its components.
\begin{align*}
\mathcal{R}_{\widetilde{\mathcal{T}}_{\boldsymbol{\pi}}}(\Theta) = \E_{(\mathbf{X},\mathbf{Y})\sim \sum_i \pi_i \mathcal{D}_i}[\ell(\dots)] = \sum_{i=1}^K \pi_i \mathcal{R}_{\mathcal{D}_i}(\Theta).
\end{align*}

\textbf{Step 3: Combining the Results.}
Substituting the result from Step 2 back into Step 1 completes the proof.
\end{proof}

\section{Detailed Proofs for MAML Analysis}\label{app:B}

This appendix contains the detailed proofs for the theoretical results presented in Section \ref{sec:maml_analysis}, focusing on the optimization dynamics.

\subsection{Proof of Lemma \ref{lem:one-step-contraction}}\label{app:proof_lem_one-step-contraction}

\lemOneStepContraction*

\begin{proof}
We aim to show that the inner objective function value contracts towards the optimum $J_{-i}^\star$ under smoothness and the PL condition (Assumption \ref{ass:smooth-pl}).

\textbf{Step 1: Apply the Descent Lemma (L-smoothness).}
Since $J_{-i}(\Theta)$ is $L$-smooth, it satisfies the Descent Lemma:
\[
J_{-i}(\Theta_i^+) \le J_{-i}(\Theta) + \langle \nabla J_{-i}(\Theta), \Theta_i^+ - \Theta \rangle + \frac{L}{2}\|\Theta_i^+ - \Theta\|^2.
\]

\textbf{Step 2: Substitute the Gradient Descent Update Rule.}
The inner adaptation step is $\Theta_i^+ = \Theta - \alpha \nabla J_{-i}(\Theta)$.
Substituting this into the inequality from Step 1:
\begin{align*}
J_{-i}(\Theta_i^+) &\le J_{-i}(\Theta) - \alpha\|\nabla J_{-i}(\Theta)\|^2 + \frac{L\alpha^2}{2}\|\nabla J_{-i}(\Theta)\|^2 \\
&= J_{-i}(\Theta) - \alpha\left(1-\frac{L\alpha}{2}\right)\|\nabla J_{-i}(\Theta)\|^2.
\end{align*}

\textbf{Step 3: Apply the Polyak--Lojasiewicz (PL) Condition.}
The learning rate condition $0<\alpha<2/L$ ensures the coefficient $\alpha(1-\frac{L\alpha}{2})$ is positive.
We utilize the PL condition:
\[
\|\nabla J_{-i}(\Theta)\|^2 \ge 2\mu\big(J_{-i}(\Theta)-J_{-i}^\star\big).
\]
Substituting this lower bound:
\begin{align*}
J_{-i}(\Theta_i^+) &\le J_{-i}(\Theta) - \alpha\left(1-\frac{L\alpha}{2}\right) \cdot 2\mu\big(J_{-i}(\Theta)-J_{-i}^\star\big).
\end{align*}

\textbf{Step 4: Rearrange to Show Contraction.}
We rearrange the terms to isolate the optimization gap at $\Theta_i^+$:
\begin{align*}
J_{-i}(\Theta_i^+) - J_{-i}^\star &\le J_{-i}(\Theta) - J_{-i}^\star - 2\mu\alpha\left(1-\frac{L\alpha}{2}\right)\big(J_{-i}(\Theta)-J_{-i}^\star\big) \\
&= \left(1 - 2\mu\alpha\left(1-\frac{L\alpha}{2}\right)\right) \big(J_{-i}(\Theta)-J_{-i}^\star\big).
\end{align*}
Defining the contraction factor $\rho(\alpha) := 1 - 2\mu\alpha(1-\frac{L\alpha}{2})$, we obtain the result. Since the subtracted term is positive, $\rho(\alpha) < 1$.
\end{proof}

\subsection{Proof of Proposition \ref{prop:DA-reduction}}\label{app:proof_prop_DA-reduction}

\propDaReduction*

\begin{proof}
We relate the contraction of the total objective $J_{-i}$ to the reduction of the DA loss component $\E\LDA^{(-i)}$.

\textbf{Step 1: Bound DA Loss by the Total Objective.}
Recall the definition of the inner objective $J_{-i}(\Theta)$:
\[
J_{-i}(\Theta) = \mathcal{R}_{\mathcal{D}_{-i}}(\Theta) + \lambda_{\mathrm{DA}} \E\LDA^{(-i)}(\theta).
\]
Assuming the classification risk $\mathcal{R}_{\mathcal{D}_{-i}}(\Theta)$ is non-negative, the DA loss term is upper-bounded by the total objective:
\[
\lambda_{\mathrm{DA}} \E\LDA^{(-i)}(\theta) \le J_{-i}(\Theta).
\]
This holds for the post-adaptation parameters $\Theta_i^+ = (\theta_i^+, \phi_i^+)$ as well:
\[
\E\LDA^{(-i)}(\theta_i^+) \le \frac{1}{\lambda_{\mathrm{DA}}} J_{-i}(\Theta_i^+).
\]

\textbf{Step 2: Apply the Contraction Result.}
From Lemma~\ref{lem:one-step-contraction}:
\[
J_{-i}(\Theta_i^+) - J_{-i}^\star \le \rho(\alpha) \big(J_{-i}(\Theta) - J_{-i}^\star\big).
\]
Rearranging this inequality to bound $J_{-i}(\Theta_i^+)$:
\begin{align*}
J_{-i}(\Theta_i^+) &\le \rho(\alpha) J_{-i}(\Theta) - \rho(\alpha) J_{-i}^\star + J_{-i}^\star \\
&= \rho(\alpha) J_{-i}(\Theta) + (1-\rho(\alpha)) J_{-i}^\star.
\end{align*}

\textbf{Step 3: Combine the Bounds.}
Substituting the upper bound for $J_{-i}(\Theta_i^+)$ from Step 2 into the inequality from Step 1:
\[
\E\LDA^{(-i)}(\theta_i^+) \le \frac{1}{\lambda_{\mathrm{DA}}} \left( \rho(\alpha) J_{-i}(\Theta) + (1-\rho(\alpha)) J_{-i}^\star \right).
\]
\[
\E\LDA^{(-i)}(\theta_i^+) \le \frac{\rho(\alpha)}{\lambda_{\mathrm{DA}}} J_{-i}(\Theta) + \frac{1-\rho(\alpha)}{\lambda_{\mathrm{DA}}} J_{-i}^\star.
\]
\end{proof}

\section{Extended Details on Theory and Methodology}\label{app:C}

This appendix provides the detailed definitions, assumptions, lemmas, and discussions that were condensed in the main paper (Sections \ref{sec:theory} and \ref{sec:methodology}).

\subsection{Detailed Setup and Assumptions (Section \ref{sec:theory})}\label{app:C_setup}

\textbf{Lipschitz Continuity and Boundedness.}
We assume the loss function $\ell$ is bounded and Lipschitz continuous with respect to the distance metric $\fd$ in the feature space $\cZ$:
\begin{equation}\label{eq:lip}
0\le \ell(g_\phi(z),c)\le B_\ell,\qquad
\big|\ell(g_\phi(z),c)-\ell(g_\phi(z'),c)\big|\le L_\ell\,\fd(z,z')\quad\forall z,z'\in\cZ.
\end{equation}

\textbf{Feature Space Diameter.}
We assume the feature space $\cZ$ is compact, implying a finite diameter bound:
\begin{equation}\label{eq:diam}
R = \sup_{z,z' \in \cZ} \fd(z,z') < \infty.
\end{equation}

\textbf{Class-Conditional Risk.}
We define the class-conditional expected risk for a domain $\mathcal{D}$ and class $c$ as $\mathcal{R}_{\mathcal{D},c}(\Theta) = \E_{\mathbf{Z} \sim P_{\mathcal{D},c}}[\ell(g_\phi(\mathbf{Z}),c)]$.

\subsection{Detailed Lemmas for Theorem \ref{thm:qs-DA} (Section \ref{sec:control_qs_gap})}\label{app:C_lemmas}

The proof of Theorem \ref{thm:qs-DA} relies on the following three lemmas, which connect the Wasserstein distance to the DA loss via an InfoNCE decomposition and a two-step transport argument. All proofs in these lemmas are provided in~\Cref{lem:nce,lem:phi,lem:two-step}.

\begin{restatable}{lemma}{lemNceIdentity}[InfoNCE-style identity]\label{lem:nce}
The expected Domain-Class Distribution Alignment loss for the query domain $\mathcal{D}_i$, denoted $\E\,\LDA^{(i)}(\theta)$, can be decomposed as:
\begin{equation}\label{eq:iden}
\E\,\LDA^{(i)}(\theta) \;=\; A_{i}(\theta)\;+\;\E\Phi(\mathbf{Z};\theta),
\end{equation}
where $A_{i}(\theta)$ is the expected distance from a query feature to its corresponding positive centroids $\{\mu_{j,c}(\theta)\}$ in the support mixture (the "positive transport cost"):
\[
A_{i}(\theta):=\sum_{c}\pi_i(c)\sum_{j\neq i}\omega_j\,\E_{\mathbf{Z}\sim P_{i,c}} \fd(f_\theta(\mathbf{Z}),\mu_{j,c}(\theta)).
\]
and $\Phi(\mathbf{Z};\theta)$ is a term related to the negative centroids (the log-partition function of the negatives).
\end{restatable}

Lemma \ref{lem:nce} shows that minimizing the DA loss simultaneously minimizes the positive transport cost $A_i$ and the negative term $\Phi(\mathbf{Z})$. To isolate $A_i$, we need to bound $\Phi(\mathbf{Z})$.

\begin{restatable}{lemma}{lemPhiBound}[Lower bound on $\Phi$ by diameter]\label{lem:phi}
The term $\Phi(\mathbf{Z};\theta)$ is lower-bounded using the diameter $R$ (Eq. \ref{eq:diam}). This leads to a crucial consequence: the expected positive transport cost $A_i$ is upper-bounded by the measurable DA loss.
\[
A_i(\theta) \le \E\,\LDA^{(i)}(\theta) + R.
\]
\end{restatable}

Now that we have connected the DA loss to the positive transport cost $A_i$ (which measures the distance from query features to support centroids), we must relate $A_i$ back to the Wasserstein distance between the underlying distributions ($P_{i,c}$ and $P_{-i,c}$).

Let $\bar{\mu}_{-i,c}:=\sum_{j\neq i}\omega_j\,\delta_{\mu_{j,c}}$ be the mixture of support-domain centroids for class $c$.
Let $S_{j,c}:=\E_{\mathbf{Z}\sim P_{j,c}} \fd(f_\theta(\mathbf{Z}),\mu_{j,c})$ be the average intra-domain spread (scatter) of features.

\begin{restatable}{lemma}{lemTwoStepTransport}[Two-step transport]\label{lem:two-step}
For any class $c$, the Wasserstein distance between $P_{i,c}$ and $P_{-i,c}$ can be bounded by applying the triangle inequality, using the centroids as intermediates:
\[
\Wone\!\big(P_{i,c},P_{-i,c}\big)
\ \le\
\underbrace{\Wone\!\big(P_{i,c},\bar{\mu}_{-i,c}\big)}_{\le \sum_{j\neq i}\omega_j\,\E_{\mathbf{Z}\sim P_{i,c}} \fd(f_\theta(\mathbf{Z}),\mu_{j,c})}
\ +\
\underbrace{\Wone\!\big(\bar{\mu}_{-i,c},P_{-i,c}\big)}_{\le \sum_{j\neq i}\omega_j\,S_{j,c}}.
\]
\end{restatable}

\textbf{Definition of $C_{\mathrm{scat}}$.}
Theorem \ref{thm:qs-DA} combines these results. The term $C_{\mathrm{scat}}^{(i)}$ mentioned in the theorem is derived from the second term of Lemma \ref{lem:two-step}, averaged over the classes according to the query prior $\pi_i(c)$:
\[
C_{\mathrm{scat}}^{(i)} := \sum_{c \in \cC} \pi_i(c) \sum_{j\neq i} \omega_j S_{j,c}.
\]

\subsection{Proof of Lemma \ref{lem:nce}}\label{app:proof_lem_nce}

\lemNceIdentity*

\begin{proof}
The proof relies on analyzing the structure of the InfoNCE loss for a single sample and then taking the expectation. We use $s_{j,c}(z) = \fd(z, \mu_{j,c})$ for brevity, noting the dependency on $\theta$.

\textbf{Step 1: Single-Sample Identity.}
Let's analyze the DA loss for a single sample $(z,c)$ from domain $i$. We define the negative partition function $\Phi(z)$:
\[
\Phi(z) = \log\sum_{(j',c')\neq(i,c)}e^{-s_{j',c'}(z)}.
\]
Using the logarithm property $\log(a/b) = \log(a) - \log(b)$, we expand the loss (assuming $K-1$ support domains and uniform weights $\omega_j = \frac{1}{K-1}$):
\begin{align*}
\text{Loss}(z,c,i) &= \sum_{j\neq i} \omega_j \left( -\log\left(e^{-s_{j,c}(z)}\right) + \log\left(\sum_{(j',c')\neq(i,c)}e^{-s_{j',c'}(z)}\right) \right) \\
&= \sum_{j\neq i} \omega_j \left( s_{j,c}(z) + \Phi(z) \right)
\end{align*}
Since $\Phi(z)$ does not depend on $j$ and $\sum_{j\neq i} \omega_j = 1$:
\begin{align*}
\text{Loss}(z,c,i) &= \sum_{j\neq i} \omega_j s_{j,c}(z) + \Phi(z).
\end{align*}

\textbf{Step 2: Taking the Expectation.}
Now, we take the expectation over samples $(\mathbf{Z},\mathbf{C})$ drawn from the query domain distribution $\mathcal{D}_i$.
\[
\E\,\LDA^{(i)}(\theta) = \E \left[ \sum_{j\neq i}\omega_j s_{j,\mathbf{C}}(\mathbf{Z}) \right] + \E[\Phi(\mathbf{Z})].
\]
Let's analyze the first term on the RHS. By definition of expectation over $\mathcal{D}_i$:
\begin{align*}
\E \left[ \sum_{j\neq i}\omega_j s_{j,\mathbf{C}}(\mathbf{Z}) \right] &= \sum_{c \in \cC} \pi_i(c) \E_{\mathbf{Z} \sim P_{i,c}} \left[ \sum_{j \neq i} \omega_j \fd(\mathbf{Z}, \mu_{j,c}) \right] \\
&= \sum_{c \in \cC} \pi_i(c) \sum_{j \neq i} \omega_j \E_{\mathbf{Z} \sim P_{i,c}} [\fd(\mathbf{Z}, \mu_{j,c})]
\end{align*}
This expression matches the definition of $A_i(\theta)$. Substituting this back, we get the desired identity.
\end{proof}

\subsection{Proof of Lemma \ref{lem:phi}}\label{app:proof_lem_phi}

\lemPhiBound*

\begin{proof}
The proof involves bounding $\Phi(z)$ using the diameter $R$ and then relating $A_i$ back to $\LDA$.

\textbf{Step 1: Bounding $\Phi(z)$.}
\begin{enumerate}
\item \textbf{Bound the Distance:} Since $z$ and $\mu_{j',c'}$ are in $\cZ$, their distance is bounded by $R$.
\[
s_{j',c'}(z) = \fd(z, \mu_{j',c'}) \le R.
\]
\item \textbf{Apply Exponential and Sum:} $e^{-s_{j',c'}(z)} \ge e^{-R}$. Let $N$ be the number of terms in the sum (number of negative pairs).
\[
\sum_{(j',c') \neq (i,c)} e^{-s_{j',c'}(z)} \ge N \cdot e^{-R}.
\]
\item \textbf{Apply Logarithm:}
\begin{align*}
\Phi(z) = \log\left(\sum_{(j',c') \neq (i,c)} e^{-s_{j',c'}(z)}\right) &\ge \log(N \cdot e^{-R}) = \log N - R.
\end{align*}
\end{enumerate}

\textbf{Step 2: Bounding $A_i$.}
\begin{enumerate}
\item \textbf{Recall Identity (Lemma \ref{lem:nce}):}
\[
A_i(\theta) = \E\,\LDA^{(i)}(\theta) - \E[\Phi(\mathbf{Z};\theta)].
\]
\item \textbf{Apply the Bound on $\Phi(\mathbf{Z})$:} We have $\Phi(\mathbf{Z}) \ge \log N - R$, so $-\Phi(\mathbf{Z}) \le R - \log N$.
\item \textbf{Substitute and Conclude:}
\[
A_i(\theta) \le \E\,\LDA^{(i)}(\theta) + R - \log N.
\]
Since $N \ge 1$, $\log N \ge 0$. Removing this non-positive term yields the simpler bound:
\[
A_i(\theta) \le \E\,\LDA^{(i)}(\theta) + R.
\]
\end{enumerate}
\end{proof}

\subsection{Proof of Lemma \ref{lem:two-step}}\label{app:proof_lem_two-step}

\lemTwoStepTransport*

\begin{proof}
The proof proceeds by applying the triangle inequality and bounding the resulting terms.

\textbf{Step 1: Triangle Inequality.}
The 1-Wasserstein distance satisfies the triangle inequality. We introduce the mixture of centroids, $\bar{\mu}_{-i,c}$, as an intermediate point.
\[
\Wone(P_{i,c}, P_{-i,c}) \le \Wone(P_{i,c}, \bar{\mu}_{-i,c}) + \Wone(\bar{\mu}_{-i,c}, P_{-i,c}).
\]

\textbf{Step 2: Bounding the First Term ($P_{i,c} \to \bar{\mu}_{-i,c}$).}
The Wasserstein distance is the minimum cost over all transport plans. We construct a specific plan: transport $\mathbf{Z} \sim P_{i,c}$ to $\mu_{j,c}$ with probability $\omega_j$. The expected cost is:
\begin{align*}
\text{Cost} &= \E_{\mathbf{Z} \sim P_{i,c}} \left[ \sum_{j \neq i} \omega_j \fd(\mathbf{Z}, \mu_{j,c}) \right] = \sum_{j \neq i} \omega_j \E_{\mathbf{Z} \sim P_{i,c}} [\fd(\mathbf{Z}, \mu_{j,c})].
\end{align*}
$\Wone(P_{i,c}, \bar{\mu}_{-i,c})$ is upper bounded by this cost.

\textbf{Step 3: Bounding the Second Term ($\bar{\mu}_{-i,c} \to P_{-i,c}$).}
Using the convexity of the Wasserstein distance:
\[
\Wone(\bar{\mu}_{-i,c}, P_{-i,c}) \le \sum_{j\neq i} \omega_j \Wone(\delta_{\mu_{j,c}}, P_{j,c}).
\]
The term $\Wone(\delta_{\mu_{j,c}}, P_{j,c})$ is the transport cost from a point mass at $\mu_{j,c}$ to the distribution $P_{j,c}$, which is exactly the intra-domain spread $S_{j,c}$.
\[
\Wone(\delta_{\mu_{j,c}}, P_{j,c}) = \E_{\mathbf{Z} \sim P_{j,c}}[\fd(\mathbf{Z}, \mu_{j,c})] = S_{j,c}.
\]
Thus, $\Wone(\bar{\mu}_{-i,c}, P_{-i,c}) \le \sum_{j\neq i} \omega_j S_{j,c}$.

\textbf{Step 4: Combining the Bounds.}
Substituting the bounds from Step 2 and Step 3 back into Step 1 completes the proof.
\end{proof}

\subsection{The Final Combined Bound (Section \ref{sec:general_arb})}\label{app:C_final_bound}

We combine the Mixture-aware bound (Theorem \ref{thm:mix}) and the final theoretical bound (Eq. \eqref{eq:qs-final}) to yield a single, comprehensive guarantee for an arbitrary target domain $\mathcal{T}$.

\begin{align*}
\mathcal{R}_{\mathcal{T}}(\Theta) &\le \sum_i \pi_i \mathcal{R}_{\mathcal{D}_i}(\Theta) + L_\ell\,\Wone\!\big(\mathcal{T},\widetilde{\mathcal{T}}_{\boldsymbol{\pi}}\big) \\
&\le \sum_i \pi_i \left\{ \underbrace{\mathcal{R}_{\mathcal{D}_{-i}}(\Theta)}_{\text{Support Error}} + \underbrace{\sum_{c} | \pi_i(c) - \pi_{-i}(c) | \mathcal{R}_{-i,c}(\Theta)}_{\text{Prior Shift Impact}} + \underbrace{L_\ell\big(\E\,\LDA^{(i)}(\theta)+C_{\scat}^{(i)}+R\big)}_{\text{Feature Shift}} \right\} \\
&\qquad + L_\ell\,\Wone\!\big(\mathcal{T},\widetilde{\mathcal{T}}_{\boldsymbol{\pi}}\big).
\end{align*}

This final inequality connects the performance on an arbitrary, unseen target domain $\mathcal{T}$ to quantities measurable on the source domains, explicitly capturing how the algorithm controls both prior shift (by minimizing $\mathcal{R}_{-i,c}(\Theta)$) and feature shift (by minimizing the DA loss).

\subsection{Detailed MAML Analysis and Discussion (Section \ref{sec:maml_analysis})}\label{app:C_maml}

\textbf{Setup for Analysis.}
The population inner objective $J_{-i}(\Theta)$ is defined as:
\begin{equation}
\label{eq:inner_objective}
J_{-i}(\Theta) := \mathcal{R}{\mathcal{D}_{-i}}(\Theta) + \lambda_{\mathrm{DA}} \E\LDA^{(-i)}(\theta).
\end{equation}

\begin{assumption}[Smoothness and Polyak--Lojasiewicz (PL) Condition]
\label{ass:smooth-pl}
$J_{-i}(\Theta)$ is $L$-smooth and satisfies the PL inequality with parameter $\mu>0$. Furthermore, the learning rate satisfies $0<\alpha<2/L$.
\end{assumption}

\begin{remark}[Discussion on the PL Condition]
The PL condition is a strong assumption, particularly for highly non-convex landscapes typical of deep neural networks. While this assumption may not hold globally in practice, it is a standard tool in optimization analysis used to establish convergence rates~\citep{karimi2016linear}. Furthermore, recent studies suggest that overparameterized networks may satisfy conditions similar to PL in relevant regions of the parameter space~\citep{dnn_pl}. We adopt this assumption to provide theoretical insights into the local optimization dynamics of the inner loop.
\end{remark}

\begin{remark}[Manifold Mixup and Optimization Dynamics]
The integration of Manifold Mixup in the inner loop serves as a vital regularization technique that potentially smooths the optimization landscape. By encouraging linear behavior in the feature space between source domains, Mixup may facilitate the contraction analyzed in Lemma \ref{lem:one-step-contraction}, potentially improving the effective smoothness constant $L$ and leading to more stable adaptation.
\end{remark}

\begin{remark}[The Connection between Source Alignment and Target Generalization]
The structure of the MLDG-style meta-learning procedure is key to its effectiveness. The inner loop performs adaptation on the source domains ($-i$), actively reducing the source alignment objective $J_{-i}$ (Proposition \ref{prop:DA-reduction}). The outer loop then minimizes the meta-objective $\min_{\Theta} \E_i [\mathcal{R}{\mathcal{D}_i}(\Theta_i^+)]$, evaluating the adapted parameters on the held-out target domain ($i$). This optimization forces the model to find an initialization $\Theta$ such that the gradient step towards better source alignment simultaneously leads to better target generalization.
\end{remark}

\subsection{Population Loss Definitions (Section \ref{sec:loss})}\label{app:C_losses}

We provide the population definitions for the loss functions used in the theoretical analysis.

\textbf{Classification Loss (Population).}
\begin{equation}
\label{eq:CE_pop}
\LCE(\theta, \phi) = \mathbb{E}_{(\mathbf{X},\mathbf{Y})}\big[\mathrm{CE}(g_\phi(f_\theta(\mathbf{X})),\mathbf{Y})\big].
\end{equation}

\textbf{Domain-Class Distribution Alignment (DA) Loss (Population).}
Let $\mu_{j,c}(\theta)$ be the centroid of class $c$ in domain $j$. The objective for domain $i$ is:
\begin{equation}
\label{eq:DA_pop}
\LDA^{(i)}(\theta) = \E_{(\mathbf{X},c) \sim \mathcal{D}_i} \left[ \frac{-1}{K-1} \sum_{j \neq i}
\log \frac{\exp\big(-\fd(f_\theta(\mathbf{X}), \mu_{j, c}(\theta))\big)}
{\sum_{(j',c') \neq (i, c)} \exp\big(-\fd(f_\theta(\mathbf{X}), \mu_{j',c'}(\theta))\big)} \right].
\end{equation}

\subsection{ERM Insufficiency and Necessity of Bounds}

A crucial implication of this analysis is the insufficiency of standard Empirical Risk Minimization (ERM), even when the target domain lies within the convex hull of the source domains. Theorem~\ref{thm:mix} bounds the target risk $R_{T}(\Theta)$ by the Source Risk ($\sum\pi_{i}R_{D_{i}}(\Theta)$) and the Distribution Shift term ($L_{l}W_{1}(T,T_{\pi})$). ERM only minimizes the Source Risk.

Critically, these two terms are coupled via the learned representation $\Theta$. The Shift Term is measured in the feature space and is therefore optimizable. However, ERM often minimizes Source Risk by learning source-specific shortcuts (e.g., spurious correlations~\cite{izmailov2022feature, chen2023understanding, deng2023robust}), which paradoxically increases the Shift Term. Therefore, controlling only the risk does not guarantee low target risk.

Furthermore, the Shift Term in Theorem~\ref{thm:mix} is intractable because it depends on the unknown target distribution $T$. \Cref{thm:qs,thm:qs-DA} are essential as they derive the DA loss as a tractable surrogate for this intractable shift term. Our meta-learning framework operationalizes this bound by treating each source domain ($\mathcal{D}_{i}$) as a virtual target. Unlike ERM, which optimizes for the source mixture, our approach optimizes for generalization across simulated unseen tasks.

\clearpage

\section{Algorithm Details}\label{app:algorithms}

This appendix provides the detailed pseudocode for the \textsc{RC-Align} training procedure (Algorithm~\ref{alg:rc_align}) and the computation of the DA loss (Algorithm~\ref{alg:compute_da_loss_corrected}).

\begin{algorithm}[H]
\caption{\textsc{RC-Align} (Robust Compound Alignment) Training Procedure}
\label{alg:rc_align}
\begin{algorithmic}[1]
\State \textbf{Input:} Source domains $\DD=\{\mathcal{D}_1, \dots, \mathcal{D}_K\}$, learning rates $\alpha$ (inner), $\eta$ (outer), $\lambda_{\mathrm{DA}}$, Mixup parameter $\alpha_{\text{mixup}}$.
\State \textbf{Initialize:} Model parameters $\Theta = (\theta, \phi)$.
\While{not converged}
    \State Sample a mini-batch $B_k$ from each domain $\mathcal{D}_k$.
    \For{$i=1$ to $K$} \Comment{LODO Episode: $i$ is the target domain}
        \State Compute support centroids $\{\mu\}_{-i}$ using $B_{-i} = \{B_k\}_{k\neq i}$ (w/o grad).
        \State Initialize inner loss accumulator: $\mathcal{L}_{\text{inner}} = 0$.
        \State Initialize adapted parameters: $\Theta' = \Theta$.

        \State \textcolor{gray}{\# --- Inner (Adaptation) Step with Manifold Mixup ---}
        \State $N_{\text{pairs}} = 0$.
        \For{$(j_1, j_2)$ in RandomPairs($\{k\}_{k\neq i}$)}
            \State $N_{\text{pairs}} \leftarrow N_{\text{pairs}} + 1$.
            \State Get batches $(X_{j_1}, Y_{j_1}), (X_{j_2}, Y_{j_2})$.
            \State $Z_{j_1} \leftarrow f_\theta(X_{j_1})$
            \State $Z_{j_2} \leftarrow f_\theta(X_{j_2})$

            \State Sample mixing coefficient $\lambda \sim \text{Beta}(\alpha_{\text{mixup}}, \alpha_{\text{mixup}})$.

            \State $Z_{\text{mix}} \leftarrow \lambda Z_{j_1} + (1-\lambda) Z_{j_2}$.

            \State \textcolor{gray}{\# Calculate mixed composite loss}
            \State $\mathcal{L}_{\mathrm{CE}, j_1} \leftarrow \LCEbatch(Z_{\text{mix}}, Y_{j_1}; \phi)$
            \State $\mathcal{L}_{\mathrm{DA}, j_1} \leftarrow \LDAbatch(Z_{\text{mix}}, Y_{j_1}, d_{j_1}; \theta, \{\mu\}_{-i})$
            \State $\widehat{\mathcal{L}}_{j_1} \leftarrow \mathcal{L}_{\mathrm{CE}, j_1} + \lambda_{\mathrm{DA}} \mathcal{L}_{\mathrm{DA}, j_1}$

            \State $\mathcal{L}_{\mathrm{CE}, j_2} \leftarrow \LCEbatch(Z_{\text{mix}}, Y_{j_2}; \phi)$
            \State $\mathcal{L}_{\mathrm{DA}, j_2} \leftarrow \LDAbatch(Z_{\text{mix}}, Y_{j_2}, d_{j_2}; \theta, \{\mu\}_{-i})$
            \State $\widehat{\mathcal{L}}_{j_2} \leftarrow \mathcal{L}_{\mathrm{CE}, j_2} + \lambda_{\mathrm{DA}} \mathcal{L}_{\mathrm{DA}, j_2}$

            \State $\mathcal{L}_{\text{mix}} \leftarrow \lambda \widehat{\mathcal{L}}_{j_1} + (1-\lambda) \widehat{\mathcal{L}}_{j_2}$
            \State $\mathcal{L}_{\text{inner}} \leftarrow \mathcal{L}_{\text{inner}} + \mathcal{L}_{\text{mix}}$.
        \EndFor
        \State $\mathcal{L}_{\text{inner}} \leftarrow \mathcal{L}_{\text{inner}} / N_{\text{pairs}}$.
        \State $\Theta' \leftarrow \Theta - \alpha\,\nabla_\Theta \mathcal{L}_{\text{inner}}$ \Comment{Adaptation (FOMAML)}

        \State \textcolor{gray}{\# --- Outer (Meta) Step ---}
        \State $\mathcal{L}_{\text{outer}} \leftarrow \LCEbatch(B_i; \Theta')$ \Comment{Evaluate on target domain}
        \State Update $\Theta$ using $\nabla_\Theta \mathcal{L}_{\text{outer}}$ with optimizer step $\eta$.
    \EndFor
\EndWhile
\end{algorithmic}
\end{algorithm}

\begin{algorithm}
\caption{Function: ComputeDALoss (Aligned with Implementation)}
\label{alg:compute_da_loss_corrected}
\begin{algorithmic}[1]
\Require Features $Z$, Labels $Y$, Domain $D_{Z}$, Centroids $\mathcal{C}$ (pre-normalized), Temperature $T$.
    \State // Normalize input features (as implemented in the Python code)
    \State $Z^{norm} \leftarrow \text{Normalize}(Z)$
    \State // Use negative L2 distance as the similarity score (logit)
    \State $Logits \leftarrow -\text{PairwiseL2Distance}(\mathcal{C}, Z^{norm})$
    
    
    \State $Logits \leftarrow Logits / T$
    
    \State // Contrastive Loss (InfoNCE style)
    \State // Define Positive Mask $M_{pos}$
    \State $M_{class} \leftarrow (\text{Labels}(\mathcal{C}) == Y)$
    \State $M_{domain} \leftarrow (\text{Domains}(\mathcal{C}) == D_{Z})$
    \State $M_{pos} \leftarrow M_{class} \text{ AND } (\neg M_{domain})$ \Comment{Same class AND different domain}
    
    \State // Define Negative Logits (Mask out positives for the denominator)
    \State $L_{neg} \leftarrow Logits$. Set $L_{neg}[M_{pos}] = -\infty$.
    
    \State // Calculate LogSumExp over negative examples
    \State $LSE \leftarrow \text{LogSumExp}(L_{neg})$
    
    \State $LogProbs \leftarrow Logits - LSE$
    \State // Calculate mean negative log probability of positive examples
    \State $\mathcal{L}_{DA} \leftarrow -\text{Mean}(LogProbs[M_{pos}])$
    \State \Return $\mathcal{L}_{DA}$
\end{algorithmic}
\end{algorithm}


\clearpage
\section{Implementation Details}\label{app:implementation}

\subsection{Dataset Details}\label{app:dataset_details}
We provide a brief overview of the datasets used in our evaluation.
\begin{itemize}
    \item \textbf{PACS}~\citep{PACS}: 9,991 images, 7 classes, 4 domains (Photo, Art painting, Cartoon, Sketch).
    \item \textbf{VLCS}~\citep{VLCS}: 10,729 images, 5 classes, 4 domains (PASCAL VOC 2007, LabelMe, Caltech-101, SUN09).
    \item \textbf{OfficeHome}~\citep{OfficeHome}: 15,588 images, 65 classes, 4 domains (Artistic, Clipart, Product, Real-World).
    \item \textbf{TerraIncognita}~\citep{TerraInc}: 24,788 images, 10 classes, 4 camera trap locations (domains).
    \item \textbf{DomainNet}~\citep{DomainNet}: Approx. 0.6 million images, 345 classes, 6 domains.
\end{itemize}

\textbf{MDLT Datasets.} The MDLT versions (PACS-MLT, VLCS-MLT, OfficeHome-MLT, TerraIncognita-MLT, and DomainNet-MLT) are from~\citet{Boda}. Crucially, the imbalance factors and the identity of majority/minority classes intentionally diverge across different domains, simulating realistic compound shifts.

\subsection{Training Protocol and Hyperparameters}\label{app:training_details}

\textbf{Architecture and Initialization.}
We utilize a standard ResNet-50 architecture. The feature extractor $f_\theta$ is initialized with weights pre-trained on ImageNet-1K. The classifier head $g_\phi$ is initialized randomly.

\textbf{General Setup and Optimization.}
We adhere to the standard training protocols of the DomainBed framework. Models are typically trained for 5,000 iterations (15,000 for DomainNet-MLT, and DomainNet126-MLT). We use the Adam optimizer~\citep{kingma2015adam} with a weight decay typically set to $5\times 10^{-4}$. The batch size is set to 32 in DomainNet, and 24 in others, following~\citet{Boda}.

\textbf{Data Augmentation.}
We employ standard data augmentations: images are resized and cropped to $224 \times 224$ using Random Resized Crop, followed by Random Horizontal Flip, and Color Jitter (brightness, contrast, saturation=0.4; hue=0.1).

\textbf{Hyperparameters and Model Selection.}
We strictly adhered to the standardized DomainBed protocol~\citep{DomainBed} for hyperparameter tuning and model selection. We employed the "Training-domain validation" approach, where models are evaluated on validation sets composed of held-out samples from the training domains.

For all methods (including baselines and \textsc{RC-Align}), hyperparameters were selected by conducting a randomized search (typically 20-50 trials per algorithm) over a predefined joint distribution of hyperparameters, as specified by the DomainBed methodology. The configuration yielding the highest accuracy on the training-domain validation set was selected. This ensures a fair comparison across all methods.

\begin{itemize}
    \item \textbf{Learning Rates:} The base learning rate $\eta$ (used for the outer loop update) and the inner loop learning rate $\alpha$ (used in the adaptation step, Eq.~\eqref{eq:inner_step}) were searched within the standard DomainBed range (e.g., log-uniform sampling between $10^{-5}$ and $10^{-3}$). When using the FOMAML approximation, $\alpha$ is often searched jointly with or set equal to $\eta$.
    \item \textbf{\textsc{RC-Align} Specific:} The alignment weight $\lambda_{\mathrm{DA}}$ was selected via a search over $\{0.1, 0.5, 1.0, 2.0\}$. The Manifold Mixup parameter $\alpha_{\text{mixup}}$ (used for the Beta distribution sampling) was searched within the standard DomainBed mixup space (e.g., $[0.1, 1.0]$).
\end{itemize}

We report the average accuracy over three independent runs with random seed.

\textbf{Experimental Environment.}
The experiments were implemented using PyTorch (v1.10+) and executed on a system equipped with NVIDIA RTX 3090 GPUs (24GB VRAM).

\clearpage

\section{Baselines for MDLT Experiments}\label{app:mdlt_baseline}

\textbf{Source of baselines.}
Unless otherwise noted, all MDLT baselines and their hyperparameters follow the public setup curated in BoDA~\citep{Boda}. We adopted the same five MDLT benchmarks and the standard training protocol so that numbers are directly comparable.

\textbf{Vanilla / Empirical Risk Minimization.}
\textbf{ERM} is trained by minimizing the average cross-entropy over pooled domains. We report ERM as implemented in DomainBed~\citep{DomainBed}.

\textbf{Domain invariance.}
\textbf{IRM}~\citep{IRM} learns representations that admit an optimal invariant classifier across environments.
\textbf{CORAL} aligns second-order feature statistics across domains~\citep{DeepCORAL}.
\textbf{MMD} minimizes the maximum mean discrepancy between domain feature distributions~\citep{MMD}.
\textbf{DANN}~\citep{DANN} uses adversarial training to make features domain-indistinguishable; 
\textbf{CDANN} conditions the discriminator on class predictions to better preserve label information~\citep{CDANN}.

\textbf{Robust optimization.}
\textbf{GroupDRO} optimizes worst-group risk under group shift to improve worst-domain accuracy~\citep{GroupDRO}.

\textbf{Meta-learning and multi-tasking.}
\textbf{MLDG} meta-learns for cross-domain generalization via simulated train/test domain splits~\citep{MLDG}.
\textbf{MTL} (marginal transfer learning) augments inputs with domain marginals and learns a domain-aware classifier~\citep{blanchard2021mtl_marginal_transfer_learning}.
\textbf{Fish} maximizes inter-domain gradient alignment to encourage updates that help across domains~\citep{Fish}.

\textbf{Augmentation and style robustness.}
\textbf{Mixup} performs convex interpolations of inputs and labels, improving regularization under compound shifts~\citep{Mixup}.
\textbf{SagNet} reduces style bias via style/content disentanglement and adversarial training~\citep{SAGNET}.

\textbf{Long-tailed (LT/MDLT) baselines.}
\textbf{Focal}~\citep{Lin2017Focal} reweights hard examples.
\textbf{CBLoss} uses the effective number of samples as class weights~\citep{Cui2019CBLoss}.
\textbf{LDAM} enforces label-distribution-aware margins~\citep{Cao2019LDAM}.
\textbf{Balanced Softmax} corrects the softmax gradient bias under label shift~\citep{Ren2020BalancedSoftmax}.
\textbf{CRT (cRT)} decouples representation learning and re-trains a balanced classifier~\citep{Kang2020cRT}.
\textbf{SSP} denotes self-supervised pretraining on the same data followed by supervised fine-tuning, which is known to boost long-tailed recognition~\citep{SSP}.

\textbf{BoDA family.}
\textbf{BoDA}, \textbf{BoDA\_r}, \textbf{BoDA-M}, and \textbf{BoDA\_{r,c}} share the same objective family that balances cross-domain alignment and within-domain calibration under class imbalance, with suffixes indicating coupled/decoupled classifier training and margin variants~\citep{Boda}.

\textbf{Our method.}
\textbf{\textsc{RC-Align} (Ours)} is evaluated under the identical data splits, model backbones, and early-stopping rules as above to ensure fair comparison.

\clearpage

\section{Detailed Results for Domain Generalization}\label{app:dg_details}

This section provides the detailed results for domain generalization benchmark.

\begin{table}[h!]
\centering
\caption{Domain generalization accuracy (\%) on the PACS dataset. Results are sourced from their original papers unless otherwise noted. $\dagger$~\cite{DomainBed}; $\ddagger$~\cite{SWAD}; $\S$~\cite{MLIR}. The best average performance is highlighted in \textbf{bold}.}
\label{tab:pacs_results}
\begin{tabular}{l|cccc|c}
\toprule
\textbf{Algorithm} & \textbf{A}      & \textbf{C}      & \textbf{P}      & \textbf{S}      & \textbf{Avg}    \\ 
\midrule
\midrule
ERM\footnotemark[2]~\citep{ERM}& 84.7            & 80.8            & 97.2            & 79.3            & 85.5            \\
IRM\footnotemark[2]~\citep{IRM}                & 84.8            & 76.4            & 96.7            & 76.1            & 83.5            \\
GroupDRO\footnotemark[2]~\citep{GroupDRO}& 83.5            & 79.1            & 96.7            & 78.3            & 84.4            \\
Mixup\footnotemark[2]~\citep{Mixup}& 86.1            & 78.9            & 97.6            & 75.8            & 84.6            \\
MLDG\footnotemark[2]~\citep{MLDG}& 85.5            & 80.1            & 97.4            & 76.6            & 84.9            \\
CORAL\footnotemark[2]~\citep{DeepCORAL}& {88.3}   & 80.0            & 97.5            & 78.8            & 86.2            \\
MMD\footnotemark[2]~\citep{MMD}& 86.1            & 79.4            & 96.6            & 76.5            & 84.6            \\
DANN\footnotemark[2]~\citep{DANN}& 86.4            & 77.4            & 97.3            & 73.5            & 83.7            \\
CDANN\footnotemark[2]~\citep{CDANN}& 84.6            & 75.5            & 96.8            & 73.5            & 82.6            \\
MTL\footnotemark[2]~\citep{blanchard2021mtl_marginal_transfer_learning}& 87.5            & 77.1            & 96.4            & 77.3            & 84.6            \\
SagNet\footnotemark[2]~\citep{SAGNET}& 87.4            & 80.7            & 97.1            & 80.0            & 86.3            \\
ARM\footnotemark[2]~\citep{ARM}& 86.8            & 76.8            & 97.4            & 79.3            & 85.1            \\
VREx\footnotemark[2]~\citep{krueger2020vrex}& 86.0            & 79.1            & 96.9            & 77.7            & 84.9            \\
RSC\footnotemark[2]~\citep{huang2020rsc}                & 85.4            & 79.7            & 97.6            & 78.2            & 85.2            \\ 
MetaReg\footnotemark[4]~\citep{MetaReg}  & -& -  & -               & -               &  83.6            \\
MLIR\footnotemark[4]~\citep{MLIR} & -               & -               & -& -  &  86.8           \\
Mixstyle\footnotemark[3]~\citep{Mixstyle}           & 86.8            & 79.0            & 96.6            & 78.5            & 85.2            \\
SWAD\footnotemark[3]~\citep{SWAD}               & 89.3            & 83.4            & 97.3            & 82.5            & 88.1            \\
PCL~\cite{PCL}& 90.2            & \textbf{83.9}            & 98.1            & {82.6}            & 88.7            \\
BoDA~\citep{Boda} & {88.2}   & {81.7}   & {97.8}   & {80.2}   & {86.9}   \\ 
SAGM~\citep{SAGM}               & 87.4            & 80.2            & 98.0            & 80.8            & 86.6            \\
iDAG~\citep{iDAG}               & \textbf{90.8}            & 83.7            & 98.0            & \textbf{82.7}            & \textbf{88.8}            \\
GMDG~\citep{GMDG}               & 84.7            & 81.7            & 97.5            & 80.5            & 85.6            \\
Arith~\citep{Arith} & 85.9 & 81.3 & 97.1 & 81.8 & 86.5 \\
\midrule
\textbf{Ours}      & 89.8\small{$\pm$0.6} & 79.9\small{$\pm$0.2} & \textbf{98.8} \small{$\pm$0.2}& 81.3\small{$\pm$1.0}& 87.5            \\
\bottomrule
\end{tabular}
\end{table}

\begin{table}[h!]
\centering
\caption{Domain generalization accuracy (\%) on the VLCS dataset. Results are sourced from their original papers unless otherwise noted. $\dagger$~\cite{DomainBed}; $\ddagger$~\cite{SWAD}; $\S$~\cite{MLIR}. The best average performance is highlighted in \textbf{bold}.}
\label{tab:vlcs_results}
\begin{tabular}{l|cccc|c}
\toprule
\textbf{Algorithm} & \textbf{C}      & \textbf{L}      & \textbf{S}      & \textbf{V}      & \textbf{Avg}    \\ 
\midrule
\midrule
ERM\footnotemark[2]~\citep{ERM}&97.7 & 64.3 & 73.4 & 74.6 & 77.5 \\
IRM\footnotemark[2]~\citep{IRM}& 98.6 & 64.9 & 73.4 & {77.3} & 78.5 \\
GroupDRO\footnotemark[2]~\citep{GroupDRO}& 97.3 & 63.4 & 69.5 & 76.7 & 76.7 \\
Mixup\footnotemark[2]~\citep{Mixup}& 98.3 & 64.8 & 72.1 & {74.3} & 77.4 \\
MLDG\footnotemark[2]~\citep{MLDG}& 97.4 & {65.2} & 71.0 & {75.3} & {77.2} \\
CORAL\footnotemark[2]~\citep{DeepCORAL}& 98.3 & 66.1 & 73.4 & 77.5 & {78.8} \\
MMD\footnotemark[2]~\citep{MMD}& 97.7 & 64.0 & 72.8 & {75.3} & 77.5 \\
DANN\footnotemark[2]~\citep{DANN}& \textbf{99.0} & 65.1 & 73.1 & 77.2 & 78.6 \\
CDANN\footnotemark[2]~\citep{CDANN}& 97.1 & 65.1 & 70.7 & 77.1 & 77.5 \\
MTL\footnotemark[2]~\citep{blanchard2021mtl_marginal_transfer_learning}& 97.8 & {64.3} & 71.5 & {75.3} & {77.2} \\
SagNet\footnotemark[2]~\citep{SAGNET}& 97.9 & 64.5 & 71.4 & 77.5 & 77.8 \\
ARM\footnotemark[2]~\citep{ARM}& 98.7 & {63.6} & 71.3 & 76.7 & 77.6 \\
VREx\footnotemark[2]~\citep{krueger2020vrex}& 98.4 & 64.4 & 74.1 & 76.2 & 78.3 \\
RSC\footnotemark[2]~\citep{huang2020rsc}& 97.9 & {62.5} & 72.3 & 75.6 & 77.1 \\
MetaReg\footnotemark[4]~\citep{MetaReg}  & -& -  & -               & -               &  76.7         \\
MLIR\footnotemark[4]~\citep{MLIR} & -               & -               & -& -  &  80.7          \\
Mixstyle\footnotemark[3]~\citep{Mixstyle}           & 98.6 & 64.5   & 72.6   &  75.7            & 77.9   \\
SWAD\footnotemark[3]~\citep{SWAD}               & 98.8      & 63.3         & 75.3           &  79.2         & 79.1        \\
PCL~\cite{PCL}& \textbf{99.0}&  63.6  & 73.8     & 75.6& 78.0  \\

BoDA~\citep{Boda}  & 98.1 & 64.5 & {74.3} & {78.0} & {78.5} \\

SAGM~\citep{SAGM} & \textbf{99.0}    &  65.2 &  75.1  & \textbf{80.7}       & 80.0          \\
iDAG~\citep{iDAG} & 98.1&  62.7 & 69.9&  77.1 & 76.9   \\
GMDG~\citep{GMDG} & 98.3 & 65.9 &73.4 &79.3 &79.2        \\
Arith~\citep{Arith} & 98.7 & 64.6& 76.3 &77.8 &79.4\\
\midrule
\textbf{Ours}      & 97.5\small{$\pm$0.4} & \textbf{68.5}\small{$\pm$0.4} & \textbf{78.1}\small{$\pm$0.8}& 79.9\small{$\pm$0.9}& \textbf{81.0}           \\
\bottomrule
\end{tabular}
\end{table}

\begin{table}[h!]
\centering
\caption{Domain generalization accuracy (\%) on the OfficeHome dataset. Results are sourced from their original papers unless otherwise noted. $\dagger$~\cite{DomainBed}; $\ddagger$~\cite{SWAD}; $\S$~\cite{MLIR}. The best average performance is highlighted in \textbf{bold}.}
\label{tab:oh_results}
\begin{tabular}{l|cccc|c}
\toprule
\textbf{Algorithm} & \textbf{A}      & \textbf{C}      & \textbf{P}      & \textbf{R}      & \textbf{Avg}    \\ 
\midrule
\midrule
ERM\footnotemark[2]~\citep{ERM} & 61.3 & 52.4 & 75.8 & 76.6 & 66.5 \\
IRM\footnotemark[2]~\citep{IRM}& 58.9 & 52.2 & 72.1 & 74.0 & 64.3 \\
GroupDRO\footnotemark[2]~\citep{GroupDRO}& 60.4 & 52.7 & 75.0 & 76.0 & 66.0 \\
Mixup\footnotemark[2]~\citep{Mixup}&  62.4 & 54.8 & 76.9 & 78.3 & 68.1 \\
MLDG\footnotemark[2]~\citep{MLDG}& 61.5 & 53.2 & 75.0 & 77.5 & 66.8 \\
CORAL\footnotemark[2]~\citep{DeepCORAL}& 65.3 & 54.4 & 76.5 & 78.4 & 68.7 \\
MMD\footnotemark[2]~\citep{MMD} & 60.4 & 53.3 & 74.3 & 77.4 & 66.3 \\
DANN\footnotemark[2]~\citep{DANN}& 59.9 & 53.0 & 73.6 & 76.9 & 65.9 \\
CDANN\footnotemark[2]~\citep{CDANN}& 61.5 & 50.4 & 74.4 & 76.6 & 65.8 \\
MTL\footnotemark[2]~\citep{blanchard2021mtl_marginal_transfer_learning}& 61.5 & 52.4 & 74.9 & 76.8 & 66.4 \\
SagNet\footnotemark[2]~\citep{SAGNET}& 63.4 & 54.8 & 75.8 & 78.3 & 68.1 \\
ARM\footnotemark[2]~\citep{ARM}&58.9 & 51.0 & 74.1 & 75.2 & 64.8 \\
VREx\footnotemark[2]~\citep{krueger2020vrex}& 60.7 & 53.0 & 75.3 & 76.6 & 66.4 \\
RSC\footnotemark[2]~\citep{huang2020rsc}& 60.7 & 51.4 & 74.8 & 75.1 & 65.5 \\
MetaReg\footnotemark[4]~\citep{MetaReg}  & -& -  & - & - &  67.6    \\
MLIR\footnotemark[4]~\citep{MLIR} & - & - & -& -  & 69.8      \\
Mixstyle\footnotemark[3]~\citep{Mixstyle} &51.1  &53.2  & 68.2    & 69.2     & 60.4 \\
SWAD\footnotemark[3]~\citep{SWAD} & 66.1 &  57.7   & 78.4     & 80.2   & 70.6     \\
PCL~\cite{PCL}& 67.3& 59.9& 78.7& 80.7& 71.6  \\

BoDA~\citep{Boda}  &{65.4} & {55.4} & {77.1} &{79.5} & {69.3} \\

SAGM~\citep{SAGM} & 65.4 &  57.0  & 78.0  & 80.0  & 70.1       \\
iDAG~\citep{iDAG} & 68.2  & \textbf{57.9} &  \textbf{79.7}  & \textbf{81.4}  & \textbf{71.8} \\
GMDG~\citep{GMDG} & - &  - &   - &  - &  70.7  \\
Arith~\citep{Arith} & 64.7 &56.3 &77.5 &79.2 &69.4 \\
\midrule
\textbf{Ours}      & \textbf{69.4}\small{$\pm$0.1} & 54.4\small{$\pm$0.4} & 79.0\small{$\pm$0.2}& 80.9\small{$\pm$0.2}& 70.9           \\
\bottomrule
\end{tabular}
\end{table}

\begin{table}[h!]
\centering
\caption{Domain generalization accuracy (\%) on the TerraIncognita dataset. Results are sourced from their original papers unless otherwise noted. $\dagger$~\cite{DomainBed}; $\ddagger$~\cite{SWAD}; $\S$~\cite{MLIR}. The best average performance is highlighted in \textbf{bold}.}
\label{tab:terra_results}
\begin{tabular}{l|cccc|c}
\toprule
\textbf{Algorithm} & \textbf{L100}      & \textbf{L38}      & \textbf{L43}      & \textbf{L46}      & \textbf{Avg}    \\ 
\midrule
\midrule
ERM\footnotemark[2]~\citep{ERM} & 49.8 & 42.1 & 56.9 & 35.7 & 46.1 \\
IRM\footnotemark[2]~\citep{IRM}& 54.6 & 39.8 & 56.2 & 39.6 & 47.6 \\
GroupDRO\footnotemark[2]~\citep{GroupDRO}& 41.2 & 38.6 & 56.7 & 36.4 & 43.2 \\
Mixup\footnotemark[2]~\citep{Mixup}&  59.6 & 42.2 & 55.9 & 33.9 & 47.9 \\
MLDG\footnotemark[2]~\citep{MLDG}& 54.2 & 44.3 & 55.6 & 36.9 & 47.7 \\
CORAL\footnotemark[2]~\citep{DeepCORAL}& 51.6 & 42.2 & 57.0 & 39.8 & 47.6 \\
MMD\footnotemark[2]~\citep{MMD} & 41.9 & 34.8 & 57.0 & 35.2 & 42.2 \\
DANN\footnotemark[2]~\citep{DANN}& 51.1 & 40.6 & 57.4 & 37.7 & 46.7 \\
CDANN\footnotemark[2]~\citep{CDANN}& 47.0 & 41.3 & 54.9 & 39.8 & 45.8 \\
MTL\footnotemark[2]~\citep{blanchard2021mtl_marginal_transfer_learning}&49.3 & 39.6 & 55.6 & 37.8 & 45.6 \\
SagNet\footnotemark[2]~\citep{SAGNET}& 53.0 & 43.0 & 57.9 & 40.4 & 48.6 \\
ARM\footnotemark[2]~\citep{ARM}& 49.3 & 38.3 & 55.8 & 38.7 & 45.5 \\
VREx\footnotemark[2]~\citep{krueger2020vrex}& 48.2 & 41.7 & 56.8 & 38.7 & 46.4 \\
RSC\footnotemark[2]~\citep{huang2020rsc}& 50.2 & 39.2 & 56.3 & 40.8 & 46.6 \\
MetaReg\footnotemark[4]~\citep{MetaReg}  & -& -  & - & - &  48.2    \\
MLIR\footnotemark[4]~\citep{MLIR} & - & - & -& -  & 51.0    \\
Mixstyle\footnotemark[3]~\citep{Mixstyle} & 54.3 & 34.1 & 55.9 & 31.7 & 44.0  \\
SWAD\footnotemark[3]~\citep{SWAD} & 55.4 & 44.9 & 59.7 & 39.9 & 50.0  \\
PCL~\cite{PCL}& 58.7 &46.3 & \textbf{60.0} & \textbf{43.6} & 52.1  \\

BoDA~\citep{Boda}  & 54.0 & 46.5 & 59.5 & 41.0 & 50.2 \\

SAGM~\citep{SAGM} & 54.8 &41.4 &57.7 & {41.3} & 48.8    \\
iDAG~\citep{iDAG} & 58.7& 35.1 &57.5& 33.0 &46.1 \\
GMDG~\citep{GMDG} & 59.8 & 45.3& 57.1 &38.2 &50.1  \\
Arith~\citep{Arith} & - &- &- &- &48.1 \\
\midrule
\textbf{Ours}      & \textbf{63.2}\small{$\pm$0.5} & \textbf{50.7}\small{$\pm$0.2} & 54.9\small{$\pm$0.3}& 41.5\small{$\pm$0.5}& \textbf{52.6}           \\
\bottomrule
\end{tabular}
\end{table}

\clearpage
\section{Detailed Results for MDLT Benchmarks}\label{app:mdlt_details}

This section provides the detailed results for MDLT benchmarks. In this section, we add experimental results from DomainNet126 \citep{DN126}.

\textbf{DomainNet126-MLT.} 
We construct DomainNet126-MLT based on the DomainNet126 dataset~\citep{DN126}, a curated subset of the large-scale DomainNet benchmark~\citep{DomainNet}, to alleviate issues arising from noisy labels in specific domains and classes.
DomainNet126-MLT contains 122,505 images from 4 domains (\textit{clipart, painting, real, sketch}) with 126 classes. 
We partitioned the dataset into training, validation, and test subsets. The validation and test sets were constructed as balanced subsets, comprising approximately 5\% (8,132 images) and 10\% (16,377 images) of the full dataset, respectively, while the remaining 97,995 images were allocated for training. The procedure for constructing the validation and test sets follows the protocol outlined in~\citep{Boda}. 

DomainNet126-MLT exhibits significant class imbalance both within and across domains. 
For example, in the \textit{real} domain, class frequencies range from 134 to 802 images per class (mean $\approx$ 553), whereas in \textit{painting} the distribution is much sparser and skewed, ranging from only 1 to 249 images per class (mean $\approx$ 81). 
This imbalance makes DomainNet126-MLT a challenging and realistic benchmark for multi-domain long-tailed recognition. It should be noted that Table~\ref{tab:dn126mlt_results} excludes SSP \citep{SSP}, as the corresponding code and model have not been fully made available.
Figure~\ref{fig:dn126_classdist_all} shows the label distributions across domains in the DomainNet126-MLT.

\begin{figure}[h]
    \centering
    \includegraphics[width=\linewidth]{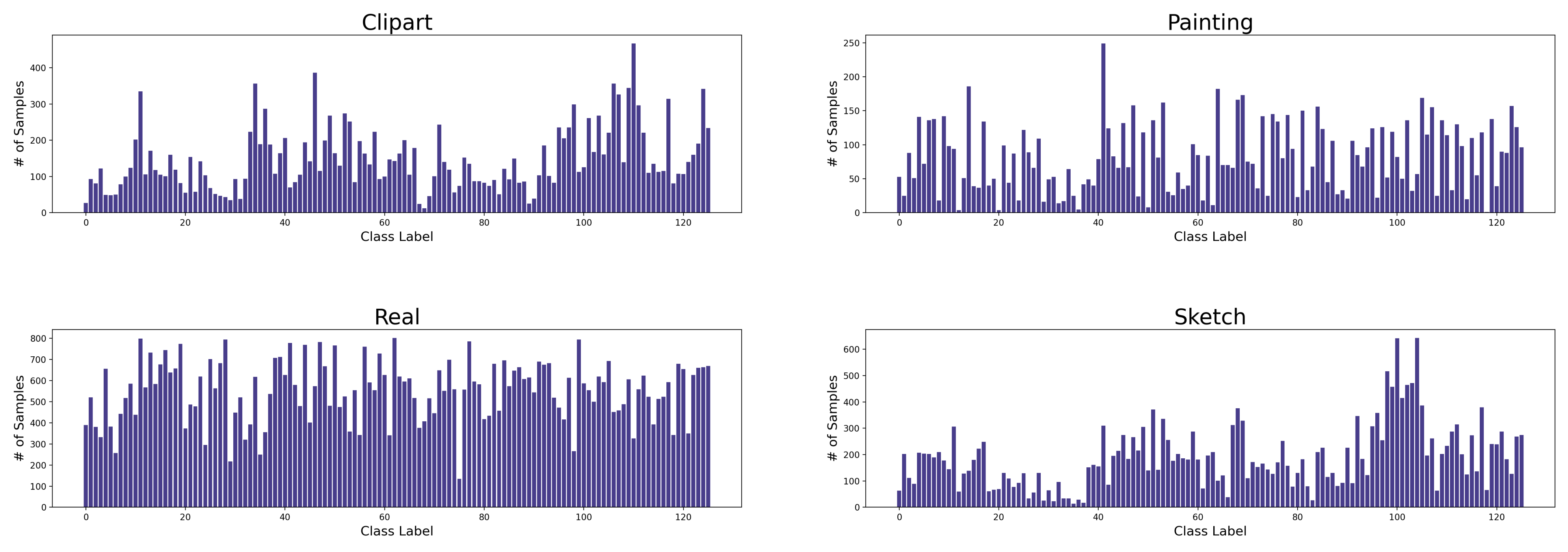}
    \caption{Class distributions across the four domains in DomainNet126-MLT. 
  The dataset exhibits a clear long-tailed distribution within and across domains.}
  \label{fig:dn126_classdist_all}
\end{figure}

\renewcommand{\thefootnote}{\fnsymbol{footnote}}
\begin{table}[h]
\centering
\setlength{\tabcolsep}{4pt}
\caption{Accuracy (\%) under the multi-domain long-tailed (MDLT) setting on \textbf{PACS-MLT}. Results are sourced from their original papers, and are reported by domain and by shot.}
\label{tab:pacsmlt_results}
\resizebox{\linewidth}{!}{%
\begin{tabular}{l|cccccc|cccc}
\toprule
\multicolumn{1}{c}{}& \multicolumn{6}{c}{\textbf{Accuracy (by domain)}} & \multicolumn{4}{c}{\textbf{Accuracy (by shot)}} \\
\cmidrule(lr){2-7}\cmidrule(lr){8-11}

\multicolumn{1}{c}{\textbf{Algorithm}} & \textbf{A} & \textbf{C} & \textbf{P} & \textbf{S} & \textbf{Average} & \multicolumn{1}{c}{\textbf{Worst}} & \textbf{Many} & \textbf{Medium} & \textbf{Few} & \textbf{Zero} \\
\midrule
\midrule
ERM~\citep{ERM} & 96.8$\pm$0.1 & 97.0$\pm$0.3 & 98.9$\pm$0.3 & 95.8$\pm$0.2 & 97.1$\pm$0.1 & 95.8$\pm$0.2 & 97.1$\pm$0.0 & 97.0$\pm$0.0 & 98.0$\pm$0.9 & - \\
IRM~\citep{IRM} & 96.8$\pm$0.1 & 96.3$\pm$0.7 & 98.7$\pm$0.2 & 95.2$\pm$0.4 & 96.7$\pm$0.2 & 95.2$\pm$0.4 & 96.8$\pm$0.2 & 96.7$\pm$0.7 & 94.7$\pm$1.4 & - \\
GroupDRO~\citep{GroupDRO} & 96.9$\pm$0.2 & 97.0$\pm$0.4 & 99.0$\pm$0.1 & 95.3$\pm$0.4 & 97.0$\pm$0.1 & 95.3$\pm$0.4 & 97.3$\pm$0.1 & 95.3$\pm$1.2 & 94.7$\pm$3.6 & - \\
Mixup~\citep{Mixup} & 96.5$\pm$0.3 & 96.9$\pm$0.7 & 98.5$\pm$0.2 & 95.1$\pm$0.2 & 96.7$\pm$0.2 & 95.1$\pm$0.2 & 97.0$\pm$0.1 & 96.7$\pm$0.3 & 91.3$\pm$2.7 & - \\
MLDG~\citep{MLDG} & 96.6$\pm$0.2 & 97.2$\pm$0.3 & 98.5$\pm$0.1 & 94.1$\pm$0.3 & 96.6$\pm$0.1 & 94.1$\pm$0.3 & 96.8$\pm$0.1 & 96.3$\pm$0.7 & 92.7$\pm$0.5 & - \\
CORAL~\citep{DeepCORAL} & 96.9$\pm$0.4 & 97.0$\pm$0.5 & 98.3$\pm$0.3 & 94.3$\pm$0.7 & 96.6$\pm$0.5 & 94.3$\pm$0.7 & 96.6$\pm$0.5 & 97.0$\pm$0.8 & 94.7$\pm$0.5 & - \\
MMD~\citep{MMD} & 96.8$\pm$0.2 & 97.1$\pm$0.4 & 97.4$\pm$0.3 & 96.3$\pm$0.3 & 96.9$\pm$0.1 & 96.2$\pm$0.2 & 96.9$\pm$0.2 & 97.0$\pm$0.0 & 96.7$\pm$0.5 & - \\
DANN~\citep{DANN} & 95.7$\pm$0.3 & 97.2$\pm$0.4 & 98.9$\pm$0.1 & 94.3$\pm$0.1 & 96.5$\pm$0.0 & 94.3$\pm$0.1 & 96.5$\pm$0.1 & 98.0$\pm$0.0 & 94.7$\pm$2.4 & - \\
CDANN~\citep{CDANN} & 95.5$\pm$0.5 & 96.7$\pm$0.2 & 97.2$\pm$0.3 & 94.9$\pm$0.5 & 96.1$\pm$0.1 & 94.5$\pm$0.2 & 96.1$\pm$0.1 & 96.3$\pm$0.5 & 94.0$\pm$0.9 & - \\
MTL~\citep{blanchard2021mtl_marginal_transfer_learning} & 96.3$\pm$0.4 & 97.9$\pm$0.3 & 98.2$\pm$0.3 & 94.6$\pm$0.7 & 96.7$\pm$0.2 & 94.5$\pm$0.6 & 96.8$\pm$0.1 & 95.3$\pm$1.7 & 97.3$\pm$1.1 & - \\
SagNet~\citep{SAGNET} & 97.0$\pm$0.2 & 97.8$\pm$0.4 & 98.9$\pm$0.1 & 95.2$\pm$0.3 & 97.2$\pm$0.1 & 95.2$\pm$0.3 & 97.4$\pm$0.1 & 96.7$\pm$0.5 & 95.3$\pm$0.5 & - \\
Fish~\citep{Fish} & 95.5$\pm$0.2 & 97.9$\pm$0.4 & 98.2$\pm$0.3 & 95.9$\pm$0.5 & 96.9$\pm$0.2 & 95.2$\pm$0.2 & 97.0$\pm$0.1 & 97.0$\pm$0.5 & 94.7$\pm$1.1 & - \\
Focal~\citep{Lin2017Focal} & 96.6$\pm$0.4 & 96.6$\pm$0.8 & 98.1$\pm$0.2 & 94.6$\pm$0.7 & 96.5$\pm$0.2 & 94.6$\pm$0.7 & 96.6$\pm$0.1 & 95.0$\pm$1.7 & 96.7$\pm$0.5 & - \\
CBLoss~\citep{Cui2019CBLoss} & 97.3$\pm$0.1 & 97.4$\pm$0.5 & 97.8$\pm$0.6 & 95.1$\pm$0.4 & 96.9$\pm$0.1 & 95.1$\pm$0.4 & 96.8$\pm$0.2 & 97.0$\pm$1.2 & \textbf{100.0}$\pm$0.0 & - \\
LDAM~\citep{Cao2019LDAM} & 96.9$\pm$0.1 & 96.6$\pm$0.6 & 97.9$\pm$0.1 & 94.7$\pm$0.2 & 96.5$\pm$0.2 & 94.7$\pm$0.2 & 96.6$\pm$0.1 & 95.7$\pm$1.4 & 96.0$\pm$0.0 & - \\
BSoftmax~\citep{Ren2020BalancedSoftmax} & 96.0$\pm$0.5 & 96.9$\pm$0.6 & 98.8$\pm$0.6 & 95.9$\pm$0.1 & 96.9$\pm$0.3 & 95.6$\pm$0.3 & 96.6$\pm$0.4 & 98.7$\pm$0.7 & 99.3$\pm$0.5 & - \\
SSP~\citep{SSP} & 96.2$\pm$0.5 & 96.8$\pm$0.2 & 98.9$\pm$0.1 & 95.7$\pm$0.3 & 96.9$\pm$0.2 & 95.4$\pm$0.4 & 96.7$\pm$0.2 & 98.3$\pm$0.5 & 98.0$\pm$0.9 & - \\
CRT~\citep{Kang2020cRT} & 95.3$\pm$0.2 & 96.7$\pm$0.1 & 98.5$\pm$0.1 & 94.9$\pm$0.1 & 96.3$\pm$0.1 & 94.9$\pm$0.1 & 96.3$\pm$0.1 & 97.3$\pm$0.3 & 94.0$\pm$0.9 & - \\
BoDAr~\citep{Boda} & 96.9$\pm$0.4 & 97.4$\pm$0.2 & 98.6$\pm$0.2 & 95.1$\pm$0.4 & 97.0$\pm$0.1 & 95.1$\pm$0.4 & 97.0$\pm$0.1 & 96.3$\pm$0.5 & 98.0$\pm$0.9 & - \\
BoDA-Mr~\citep{Boda} & 96.6$\pm$0.2 & 98.0$\pm$0.2 & 99.1$\pm$0.2 & 94.9$\pm$0.1 & 97.1$\pm$0.1 & 94.9$\pm$0.1 & 97.3$\pm$0.1 & 96.3$\pm$0.5 & 96.0$\pm$0.0 & - \\
BoDAr,c~\citep{Boda} & 96.3$\pm$0.1 & 97.4$\pm$0.5 & 99.4$\pm$0.3 & 95.7$\pm$0.3 & 97.2$\pm$0.1 & 95.7$\pm$0.3 & 97.4$\pm$0.1 & 97.0$\pm$0.0 & 94.7$\pm$1.1 & - \\
BoDA-Mr,c~\citep{Boda} & 96.3$\pm$0.4 & 97.7$\pm$0.2 & 98.1$\pm$0.4 & 96.4$\pm$0.2 & 97.1$\pm$0.2 & 96.3$\pm$0.1 & 97.1$\pm$0.0 & 97.0$\pm$0.8 & 96.0$\pm$0.0 & - \\
\midrule
\textbf{Ours} & \textbf{98.3}$\pm$0.3 & \textbf{98.2}$\pm$0.1 & \textbf{99.6}$\pm$0.3 & \textbf{97.4}$\pm$0.5 & \textbf{98.4}$\pm$0.3 & \textbf{97.4}$\pm$0.5 & \textbf{98.3}$\pm$0.3 & \textbf{99.7}$\pm$0.3 & 98.7$\pm$0.7 & - \\
\bottomrule
\end{tabular}}
\end{table}

\renewcommand{\thefootnote}{\fnsymbol{footnote}}
\begin{table}[h]
\centering
\setlength{\tabcolsep}{4pt}
\caption{Accuracy (\%) under the multi-domain long-tailed (MDLT) setting on \textbf{VLCS-MLT}. Results are sourced from their original papers, and are reported by domain and by shot.}
\label{tab:vlcsmlt_results}
\resizebox{\linewidth}{!}{%
\begin{tabular}{l|cccccc|cccc}
\toprule
\multicolumn{1}{c}{}& \multicolumn{6}{c}{\textbf{Accuracy (by domain)}} & \multicolumn{4}{c}{\textbf{Accuracy (by shot)}} \\
\cmidrule(lr){2-7}\cmidrule(lr){8-11}
\multicolumn{1}{c}{\textbf{Algorithm}} & \textbf{C} & \textbf{L} & \textbf{S} & \textbf{V} & \textbf{Average} & \multicolumn{1}{c}{\textbf{Worst}} & \textbf{Many} & \textbf{Medium} & \textbf{Few} & \textbf{Zero} \\
\midrule
\midrule
ERM~\citep{ERM} & 99.3$\pm$0.3 & 53.6$\pm$1.1 & 65.9$\pm$1.2 & 86.4$\pm$0.7 & 76.3$\pm$0.4 & 53.6$\pm$1.1 & 84.6$\pm$0.5 & 76.6$\pm$0.4 & - & 32.9$\pm$0.4 \\
IRM~\citep{IRM} & 99.1$\pm$0.4 & 52.3$\pm$0.7 & 68.8$\pm$1.4 & 86.0$\pm$0.3 & 76.5$\pm$0.2 & 52.3$\pm$0.7 & 85.3$\pm$0.6 & 75.5$\pm$1.0 & - & 33.5$\pm$1.0 \\
GroupDRO~\citep{GroupDRO} & 98.7$\pm$0.3 & 54.1$\pm$1.3 & 67.5$\pm$1.5 & 86.7$\pm$0.3 & 76.7$\pm$0.4 & 54.1$\pm$1.3 & 85.3$\pm$0.9 & 76.2$\pm$1.0 & - & 34.5$\pm$2.0 \\
Mixup~\citep{Mixup} & 99.3$\pm$0.3 & 52.7$\pm$1.3 & 66.1$\pm$0.0 & 85.3$\pm$1.1 & 75.9$\pm$0.1 & 52.7$\pm$1.3 & 84.4$\pm$0.2 & 77.1$\pm$0.6 & - & 29.2$\pm$1.4 \\
MLDG~\citep{MLDG} & 99.3$\pm$0.3 & 53.6$\pm$0.5 & 68.3$\pm$0.4 & 86.4$\pm$0.5 & 76.9$\pm$0.2 & 53.6$\pm$0.5 & 84.9$\pm$0.3 & 77.5$\pm$1.0 & - & 34.4$\pm$0.9 \\
CORAL~\citep{DeepCORAL} & 99.3$\pm$0.3 & 51.6$\pm$0.7 & 67.5$\pm$1.8 & 85.3$\pm$0.9 & 75.9$\pm$0.5 & 51.6$\pm$0.7 & 84.3$\pm$0.6 & 75.5$\pm$0.5 & - & 34.5$\pm$0.8 \\
MMD~\citep{MMD} & 99.6$\pm$0.2 & 53.4$\pm$0.3 & 65.6$\pm$0.8 & 86.7$\pm$1.1 & 76.3$\pm$0.6 & 53.4$\pm$0.3 & 84.5$\pm$0.8 & 77.1$\pm$0.5 & - & 32.7$\pm$0.3 \\
DANN~\citep{DANN} & 99.6$\pm$0.2 & 54.1$\pm$0.3 & 69.9$\pm$0.2 & 86.7$\pm$0.0 & 77.5$\pm$0.1 & 54.1$\pm$0.3 & 85.9$\pm$0.5 & 76.0$\pm$0.4 & - & 38.0$\pm$2.3 \\
CDANN~\citep{CDANN} & 99.6$\pm$0.4 & 53.6$\pm$0.4 & 67.5$\pm$0.6 & 85.8$\pm$0.8 & 76.6$\pm$0.4 & 53.6$\pm$0.4 & 84.4$\pm$0.7 & 77.3$\pm$0.8 & - & 35.0$\pm$0.8 \\
MTL~\citep{blanchard2021mtl_marginal_transfer_learning} & 99.1$\pm$0.2 & 52.9$\pm$0.5 & 66.7$\pm$0.4 & 86.7$\pm$0.6 & 76.3$\pm$0.3 & 52.9$\pm$0.5 & 84.8$\pm$0.9 & 76.2$\pm$0.6 & - & 33.3$\pm$1.4 \\
SagNet~\citep{SAGNET} & 99.6$\pm$0.4 & 52.3$\pm$0.2 & 67.2$\pm$0.2 & 86.2$\pm$1.0 & 76.3$\pm$0.2 & 52.3$\pm$0.2 & 85.3$\pm$0.3 & 75.1$\pm$0.2 & - & 32.9$\pm$0.3 \\
Fish~\citep{Fish} & 98.7$\pm$0.3 & 54.3$\pm$0.4 & 69.4$\pm$0.8 & 87.6$\pm$0.4 & 77.5$\pm$0.3 & 54.3$\pm$0.4 & 86.2$\pm$0.5 & 76.0$\pm$0.4 & - & 35.6$\pm$2.2 \\
Focal~\citep{Lin2017Focal} & 99.1$\pm$0.4 & 52.3$\pm$0.2 & 66.1$\pm$0.8 & 84.9$\pm$0.2 & 75.6$\pm$0.4 & 52.3$\pm$0.2 & 84.0$\pm$0.2 & 75.5$\pm$0.6 & - & 32.7$\pm$0.9 \\
CBLoss~\citep{Cui2019CBLoss} & 99.1$\pm$0.2 & 52.5$\pm$0.5 & 68.5$\pm$1.0 & 87.1$\pm$1.0 & 76.8$\pm$0.3 & 52.5$\pm$0.5 & 84.8$\pm$0.7 & 77.5$\pm$1.4 & - & 33.2$\pm$1.6 \\
LDAM~\citep{Cao2019LDAM} & 98.9$\pm$0.2 & 52.9$\pm$0.2 & 69.4$\pm$1.4 & 88.0$\pm$1.3 & 77.5$\pm$0.1 & 52.9$\pm$0.2 & 86.5$\pm$0.4 & 75.5$\pm$0.5 & - & 35.2$\pm$0.6 \\
BSoftmax~\citep{Ren2020BalancedSoftmax} & 99.3$\pm$0.3 & 52.9$\pm$0.9 & 68.0$\pm$0.2 & 86.7$\pm$0.8 & 76.7$\pm$0.5 & 52.9$\pm$0.9 & 84.4$\pm$0.9 & 78.2$\pm$0.6 & - & 34.3$\pm$0.9 \\
SSP~\citep{SSP} & 99.1$\pm$0.2 & 52.3$\pm$1.0 & 68.0$\pm$0.2 & 85.1$\pm$0.4 & 76.1$\pm$0.3 & 52.3$\pm$1.0 & 83.8$\pm$0.3 & 76.0$\pm$1.2 & - & 37.1$\pm$0.7 \\
CRT~\citep{Kang2020cRT} & 99.6$\pm$0.3 & 51.4$\pm$0.3 & 66.9$\pm$0.8 & 86.9$\pm$0.4 & 76.3$\pm$0.2 & 51.4$\pm$0.3 & 84.5$\pm$0.1 & 77.3$\pm$0.0 & - & 31.7$\pm$1.0 \\
BoDAr~\citep{Boda} & 99.3$\pm$0.3 & 51.4$\pm$0.3 & 70.2$\pm$0.4 & 86.7$\pm$0.3 & 76.9$\pm$0.5 & 51.4$\pm$0.3 & 85.3$\pm$0.3 & 77.3$\pm$0.2 & - & 33.3$\pm$0.5 \\
BoDA-Mr~\citep{Boda} & \textbf{100.0}$\pm$0.0 & 53.4$\pm$0.3 & 68.5$\pm$0.4 & 88.0$\pm$0.8 & 77.5$\pm$0.3 & 53.4$\pm$0.3 & 85.8$\pm$0.2 & 77.3$\pm$0.2 & - & 35.7$\pm$0.7 \\
BoDAr,c~\citep{Boda} & 99.3$\pm$0.3 & 53.4$\pm$0.3 & 68.5$\pm$0.4 & 88.0$\pm$0.4 & 77.3$\pm$0.2 & 53.4$\pm$0.3 & 85.3$\pm$0.3 & 78.0$\pm$0.2 & - & 38.6$\pm$0.7 \\
BoDA-Mr,c~\citep{Boda} & \textbf{100.0}$\pm$0.0 & 55.4$\pm$0.5 & \textbf{72.6}$\pm$0.3 & 84.7$\pm$0.5 & 78.2$\pm$0.4 & 55.4$\pm$0.5 & 85.3$\pm$0.3 & 79.3$\pm$0.6 & - & \textbf{43.3}$\pm$1.1 \\
\midrule
\textbf{Ours} & 99.8$\pm$0.2 & \textbf{57.4}$\pm$0.8 & 70.4$\pm$1.0 & \textbf{88.7}$\pm$0.8 & \textbf{79.1}$\pm$0.2 & \textbf{57.4}$\pm$0.8 & \textbf{86.9}$\pm$0.9 & \textbf{79.5}$\pm$1.2 & - & 39.4$\pm$1.6 \\
\bottomrule
\end{tabular}}
\end{table}

\renewcommand{\thefootnote}{\fnsymbol{footnote}}
\begin{table}[h]
\centering
\setlength{\tabcolsep}{4pt}
\caption{Accuracy (\%) under the multi-domain long-tailed (MDLT) setting on \textbf{OfficeHome}. Results are sourced from their original papers, and are reported by domain and by shot.}
\label{tab:ohmlt_results}
\resizebox{\linewidth}{!}{%
\begin{tabular}{l|cccccc|cccc}
\toprule
\multicolumn{1}{c}{}& \multicolumn{6}{c}{\textbf{Accuracy (by domain)}} & \multicolumn{4}{c}{\textbf{Accuracy (by shot)}} \\
\cmidrule(lr){2-7}\cmidrule(lr){8-11}
\multicolumn{1}{c}{\textbf{Algorithm}} & \textbf{A} & \textbf{C} & \textbf{P} & \textbf{R} & \textbf{Average} & \multicolumn{1}{c}{\textbf{Worst}} & \textbf{Many} & \textbf{Medium} & \textbf{Few} & \textbf{Zero} \\
\midrule
\midrule
ERM~\citep{ERM} & 71.3$\pm$0.1 & 78.4$\pm$0.2 & 89.6$\pm$0.3 & 83.3$\pm$0.2 & 80.7$\pm$0.0 & 71.3$\pm$0.1 & 87.8$\pm$0.2 & 81.0$\pm$0.2 & 63.1$\pm$0.1 & 63.3$\pm$7.2 \\
IRM~\citep{IRM} & 70.7$\pm$0.2 & 78.5$\pm$0.8 & 89.4$\pm$0.5 & 83.8$\pm$0.6 & 80.6$\pm$0.4 & 70.7$\pm$0.2 & 87.6$\pm$0.4 & 81.5$\pm$0.4 & 61.1$\pm$0.9 & 56.7$\pm$1.4 \\
GroupDRO~\citep{GroupDRO} & 68.7$\pm$0.9 & 79.0$\pm$0.2 & 89.4$\pm$0.4 & 83.3$\pm$0.5 & 80.1$\pm$0.3 & 68.7$\pm$0.9 & 88.1$\pm$0.2 & 80.8$\pm$0.4 & 59.8$\pm$1.2 & 51.7$\pm$3.6 \\
Mixup~\citep{Mixup} & 72.3$\pm$0.6 & 79.1$\pm$0.4 & 89.7$\pm$0.1 & 83.9$\pm$0.2 & 81.2$\pm$0.2 & 72.3$\pm$0.6 & 87.9$\pm$0.4 & 81.8$\pm$0.1 & 64.1$\pm$0.4 & 60.0$\pm$4.1 \\
MLDG~\citep{MLDG} & 70.2$\pm$0.6 & 78.2$\pm$0.5 & 89.4$\pm$0.4 & 83.7$\pm$0.3 & 80.4$\pm$0.2 & 70.2$\pm$0.6 & 87.1$\pm$0.1 & 81.3$\pm$0.3 & 61.3$\pm$1.0 & 61.7$\pm$1.4 \\
CORAL~\citep{DeepCORAL} & 72.7$\pm$0.6 & 80.9$\pm$0.3 & 89.9$\pm$0.2 & 84.2$\pm$0.4 & 81.9$\pm$0.1 & 72.7$\pm$0.6 & 87.9$\pm$0.1 & 83.0$\pm$0.1 & 63.5$\pm$0.7 & 65.0$\pm$2.4 \\
MMD~\citep{MMD} & 67.7$\pm$0.8 & 77.8$\pm$0.2 & 87.4$\pm$0.5 & 80.6$\pm$0.4 & 78.4$\pm$0.4 & 67.7$\pm$0.8 & 85.2$\pm$0.2 & 79.4$\pm$0.7 & 58.8$\pm$0.4 & 56.7$\pm$3.6 \\
DANN~\citep{DANN} & 70.2$\pm$0.9 & 77.3$\pm$0.3 & 87.3$\pm$0.5 & 82.1$\pm$0.4 & 79.2$\pm$0.2 & 70.2$\pm$0.9 & 86.2$\pm$0.1 & 80.0$\pm$0.1 & 60.3$\pm$1.1 & 61.7$\pm$5.9 \\
CDANN~\citep{CDANN} & 69.4$\pm$0.3 & 77.2$\pm$0.3 & 87.7$\pm$0.2 & 81.5$\pm$0.3 & 79.0$\pm$0.2 & 69.4$\pm$0.3 & 86.4$\pm$0.6 & 79.8$\pm$0.1 & 58.9$\pm$0.8 & 50.0$\pm$4.7 \\
MTL~\citep{blanchard2021mtl_marginal_transfer_learning} & 69.8$\pm$0.6 & 77.6$\pm$0.3 & 87.9$\pm$0.1 & 82.4$\pm$0.3 & 79.5$\pm$0.2 & 69.8$\pm$0.6 & 87.3$\pm$0.3 & 79.8$\pm$0.2 & 61.1$\pm$0.2 & 51.7$\pm$2.7 \\
SagNet~\citep{SAGNET} & 70.5$\pm$0.5 & 79.6$\pm$0.5 & 89.3$\pm$0.4 & 83.9$\pm$0.1 & 80.9$\pm$0.1 & 70.5$\pm$0.5 & 87.8$\pm$0.4 & 81.9$\pm$0.1 & 61.2$\pm$0.9 & 56.7$\pm$3.6 \\
Fish~\citep{Fish} & 71.3$\pm$0.7 & 79.1$\pm$0.1 & 90.2$\pm$0.6 & 84.7$\pm$0.4 & 81.3$\pm$0.3 & 71.3$\pm$0.7 & 88.2$\pm$0.2 & 81.9$\pm$0.3 & 63.2$\pm$0.8 & 61.7$\pm$1.4 \\
Focal~\citep{Lin2017Focal} & 67.6$\pm$0.4 & 76.6$\pm$0.8 & 87.1$\pm$0.5 & 80.2$\pm$0.3 & 77.9$\pm$0.0 & 67.6$\pm$0.4 & 86.5$\pm$0.3 & 78.3$\pm$0.1 & 57.4$\pm$0.3 & 46.7$\pm$3.6 \\
CBLoss~\citep{Cui2019CBLoss} & 69.5$\pm$0.7 & 78.7$\pm$0.3 & 88.9$\pm$0.4 & 82.2$\pm$0.1 & 79.8$\pm$0.2 & 69.5$\pm$0.7 & 86.6$\pm$0.4 & 80.6$\pm$0.2 & 61.1$\pm$1.4 & 65.0$\pm$2.4 \\
LDAM~\citep{Cao2019LDAM} & 69.9$\pm$0.5 & 78.9$\pm$0.4 & 89.4$\pm$0.3 & 83.0$\pm$0.4 & 80.3$\pm$0.2 & 69.9$\pm$0.5 & 87.1$\pm$0.2 & 81.3$\pm$0.3 & 61.1$\pm$0.2 & 51.7$\pm$2.7 \\
BSoftmax~\citep{Ren2020BalancedSoftmax} & 70.9$\pm$0.5 & 78.7$\pm$0.2 & 89.0$\pm$0.8 & 83.0$\pm$0.3 & 80.4$\pm$0.2 & 70.9$\pm$0.5 & 86.7$\pm$0.5 & 81.3$\pm$0.3 & 62.4$\pm$1.0 & 60.0$\pm$4.1 \\
SSP~\citep{SSP} & 71.1$\pm$0.3 & 79.6$\pm$0.8 & 89.4$\pm$0.3 & 84.2$\pm$0.2 & 81.1$\pm$0.3 & 71.1$\pm$0.3 & 87.3$\pm$0.6 & 82.3$\pm$0.3 & 61.6$\pm$0.7 & 63.3$\pm$1.4 \\
CRT~\citep{Kang2020cRT} & 72.5$\pm$0.2 & 79.6$\pm$0.2 & 88.9$\pm$0.1 & 83.6$\pm$0.2 & 81.2$\pm$0.0 & 72.5$\pm$0.2 & 87.7$\pm$0.1 & 81.8$\pm$0.1 & 64.0$\pm$0.1 & 65.0$\pm$2.4 \\
BoDAr~\citep{Boda} & 71.8$\pm$0.1 & 80.3$\pm$0.3 & 89.1$\pm$0.4 & 84.6$\pm$0.2 & 81.5$\pm$0.1 & 71.8$\pm$0.1 & 87.7$\pm$0.2 & 82.3$\pm$0.1 & 64.2$\pm$0.3 & 63.3$\pm$1.4 \\
BoDA-Mr~\citep{Boda} & 71.6$\pm$0.2 & 80.5$\pm$0.3 & 89.2$\pm$0.2 & 85.7$\pm$0.4 & 81.9$\pm$0.2 & 71.6$\pm$0.2 & 87.3$\pm$0.3 & 83.4$\pm$0.2 & 62.3$\pm$0.3 & 65.0$\pm$2.4 \\
BoDAr,c~\citep{Boda} & 72.3$\pm$0.3 & 80.8$\pm$0.2 & 89.4$\pm$0.4 & 86.3$\pm$0.3 & 82.3$\pm$0.1 & 72.3$\pm$0.3 & 87.1$\pm$0.2 & 83.9$\pm$0.3 & 63.2$\pm$0.2 & 65.0$\pm$2.4 \\
BoDA-Mr,c~\citep{Boda} & 72.3$\pm$0.3 & \textbf{81.5}$\pm$0.4 & 89.5$\pm$0.3 & 85.8$\pm$0.2 & 82.4$\pm$0.2 & 72.3$\pm$0.3 & 87.7$\pm$0.1 & 83.9$\pm$0.6 & 64.2$\pm$0.3 & \textbf{66.7}$\pm$2.7 \\
\midrule
\textbf{Ours} & \textbf{75.4}$\pm$0.1 & 79.2$\pm$0.2 & \textbf{92.1}$\pm$0.1 & \textbf{86.7}$\pm$0.3 & \textbf{83.3}$\pm$0.0 & \textbf{75.4}$\pm$0.1 & \textbf{89.4}$\pm$0.2 & \textbf{84.1}$\pm$0.1 & \textbf{67.1}$\pm$0.2 & 56.7$\pm$3.3 \\
\bottomrule
\end{tabular}}
\end{table}

\renewcommand{\thefootnote}{\fnsymbol{footnote}}
\begin{table}[h]
\centering
\setlength{\tabcolsep}{4pt}
\caption{Accuracy (\%) under the multi-domain long-tailed (MDLT) setting on \textbf{TerraInc-MLT}. Results are sourced from their original papers, and are reported by domain and by shot.}
\label{tab:terramlt_results}
\resizebox{\linewidth}{!}{%
\begin{tabular}{l|cccccc|cccc}
\toprule
\multicolumn{1}{c}{}& \multicolumn{6}{c}{\textbf{Accuracy (by domain)}} & \multicolumn{4}{c}{\textbf{Accuracy (by shot)}} \\
\cmidrule(lr){2-7}\cmidrule(lr){8-11}
\multicolumn{1}{c}{\textbf{Algorithm}} & \textbf{L100} & \textbf{L38} & \textbf{L43} & \textbf{L46} & \textbf{Average} & \multicolumn{1}{c}{\textbf{Worst}} & \textbf{Many} & \textbf{Medium} & \textbf{Few} & \textbf{Zero} \\
\midrule\midrule
ERM~\citep{ERM} & 80.3$\pm$1.3 & 71.2$\pm$0.7 & 82.2$\pm$0.3 & 67.4$\pm$0.3 & 75.3$\pm$0.3 & 67.4$\pm$0.3 & 85.6$\pm$0.8 & 69.6$\pm$3.2 & 66.1$\pm$2.4 & 14.4$\pm$2.8 \\
IRM~\citep{IRM} & 78.2$\pm$0.9 & 69.6$\pm$2.0 & 81.1$\pm$0.7 & 64.3$\pm$1.3 & 73.3$\pm$0.7 & 64.3$\pm$1.3 & 83.5$\pm$0.6 & 70.0$\pm$1.8 & 58.3$\pm$3.4 & 20.1$\pm$1.4 \\
GroupDRO~\citep{GroupDRO} & 68.3$\pm$1.0 & 68.8$\pm$1.3 & 82.6$\pm$0.2 & 68.1$\pm$0.8 & 72.0$\pm$0.4 & 66.6$\pm$0.2 & 84.7$\pm$1.1 & 64.6$\pm$4.7 & 38.9$\pm$1.2 & 13.5$\pm$1.1 \\
Mixup~\citep{Mixup} & 75.4$\pm$1.4 & 70.2$\pm$1.3 & 78.3$\pm$0.6 & 60.4$\pm$1.1 & 71.1$\pm$0.7 & 60.4$\pm$1.1 & 83.2$\pm$0.7 & 60.0$\pm$0.6 & 56.1$\pm$3.0 & 12.2$\pm$2.1 \\
MLDG~\citep{MLDG} & 82.3$\pm$0.9 & 73.5$\pm$2.0 & 83.8$\pm$1.4 & 66.9$\pm$0.5 & 76.6$\pm$0.2 & 66.9$\pm$0.5 & 86.1$\pm$0.6 & 73.8$\pm$3.9 & 70.6$\pm$3.7 & 18.8$\pm$2.4 \\
CORAL~\citep{DeepCORAL} & 81.6$\pm$1.0 & 72.0$\pm$0.6 & 84.2$\pm$0.2 & 67.8$\pm$0.9 & 76.4$\pm$0.5 & 67.8$\pm$0.9 & 86.3$\pm$0.3 & 77.5$\pm$3.1 & 66.1$\pm$2.0 & 11.0$\pm$1.4 \\
MMD~\citep{MMD} & 78.9$\pm$0.6 & 68.8$\pm$1.0 & 81.9$\pm$0.9 & 63.7$\pm$1.1 & 73.3$\pm$0.4 & 63.7$\pm$1.1 & 84.0$\pm$0.4 & 67.9$\pm$2.7 & 60.6$\pm$1.6 & 13.6$\pm$2.6 \\
DANN~\citep{DANN} & 74.1$\pm$0.8 & 63.1$\pm$1.9 & 75.9$\pm$0.2 & 61.5$\pm$0.9 & 68.7$\pm$0.9 & 61.1$\pm$1.0 & 79.6$\pm$1.2 & 62.5$\pm$8.1 & 48.9$\pm$2.8 & 13.3$\pm$1.1 \\
CDANN~\citep{CDANN} & 73.0$\pm$1.3 & 67.8$\pm$2.0 & 75.0$\pm$0.6 & 65.2$\pm$1.1 & 70.3$\pm$0.5 & 63.9$\pm$1.0 & 83.5$\pm$0.8 & 50.0$\pm$4.2 & 43.9$\pm$4.7 & 20.4$\pm$3.1 \\
MTL~\citep{blanchard2021mtl_marginal_transfer_learning} & 79.4$\pm$0.8 & 70.8$\pm$0.6 & 81.9$\pm$0.8 & 67.8$\pm$1.4 & 75.0$\pm$0.7 & 67.7$\pm$1.4 & 85.2$\pm$0.7 & 73.8$\pm$1.6 & 61.1$\pm$2.8 & 12.4$\pm$4.0 \\
SagNet~\citep{SAGNET} & 79.4$\pm$1.8 & 71.2$\pm$0.7 & 83.4$\pm$2.4 & 66.5$\pm$2.1 & 75.1$\pm$1.6 & 66.5$\pm$2.1 & 85.5$\pm$0.9 & 77.1$\pm$5.0 & 57.8$\pm$4.3 & 13.0$\pm$3.4 \\
Fish~\citep{Fish} & 80.1$\pm$1.9 & 70.2$\pm$0.2 & 84.4$\pm$0.9 & 66.3$\pm$0.5 & 75.3$\pm$0.5 & 66.3$\pm$0.5 & 85.8$\pm$0.2 & 73.3$\pm$3.9 & 61.1$\pm$3.0 & 13.7$\pm$3.3 \\
Focal~\citep{Lin2017Focal} & 80.9$\pm$0.7 & 71.6$\pm$1.6 & 84.4$\pm$1.3 & 66.1$\pm$1.7 & 75.7$\pm$0.4 & 65.3$\pm$1.1 & 85.7$\pm$0.3 & 76.2$\pm$3.9 & 68.9$\pm$3.2 & 12.6$\pm$1.9 \\
CBLoss~\citep{Cui2019CBLoss} & 84.9$\pm$0.6 & 78.0$\pm$1.2 & 80.7$\pm$0.3 & 68.3$\pm$2.0 & 78.0$\pm$0.4 & 68.3$\pm$2.0 & 85.0$\pm$0.1 & 89.2$\pm$1.2 & 83.9$\pm$2.5 & 9.3$\pm$3.9 \\
LDAM~\citep{Cao2019LDAM} & 83.0$\pm$0.9 & 70.6$\pm$0.6 & 81.3$\pm$1.1 & 64.1$\pm$1.4 & 74.7$\pm$0.9 & 64.1$\pm$1.4 & 85.1$\pm$0.6 & 70.8$\pm$3.5 & 67.8$\pm$1.2 & 11.1$\pm$2.4 \\
BSoftmax~\citep{Ren2020BalancedSoftmax} & 83.5$\pm$2.1 & 75.5$\pm$0.4 & 82.1$\pm$0.7 & 65.6$\pm$1.3 & 76.7$\pm$1.0 & 65.6$\pm$1.3 & 83.4$\pm$0.8 & 90.8$\pm$0.9 & 78.3$\pm$3.9 & 12.6$\pm$2.4 \\
SSP~\citep{SSP} & 82.6$\pm$1.3 & 80.7$\pm$1.8 & 83.2$\pm$0.6 & 67.3$\pm$0.4 & 78.5$\pm$0.7 & 67.3$\pm$0.4 & 85.5$\pm$1.0 & 87.8$\pm$0.9 & 82.6$\pm$1.2 & 13.2$\pm$2.8 \\
CRT~\citep{Kang2020cRT} & 89.0$\pm$0.1 & 81.8$\pm$0.3 & 85.8$\pm$0.3 & 70.0$\pm$0.4 & 81.6$\pm$0.1 & 70.0$\pm$0.4 & 89.7$\pm$0.2 & 90.4$\pm$0.3 & 83.9$\pm$0.5 & 12.9$\pm$0.0 \\
BoDAr~\citep{Boda} & 86.7$\pm$0.7 & 74.1$\pm$1.1 & 85.2$\pm$0.7 & 68.5$\pm$0.3 & 78.6$\pm$0.4 & 68.5$\pm$0.3 & 86.4$\pm$0.1 & 85.0$\pm$1.0 & 80.0$\pm$0.9 & 13.7$\pm$2.1 \\
BoDA-Mr~\citep{Boda} & 87.8$\pm$0.9 & 76.5$\pm$0.9 & 82.2$\pm$0.3 & 71.3$\pm$0.4 & 79.4$\pm$0.6 & 71.3$\pm$0.4 & 88.4$\pm$0.3 & 76.2$\pm$2.7 & 88.3$\pm$1.6 & 14.4$\pm$1.4 \\
BoDAr,c~\citep{Boda} & 88.3$\pm$0.6 & 82.9$\pm$0.5 & \textbf{89.3}$\pm$0.9 & 68.5$\pm$0.6 & 82.3$\pm$0.3 & 68.5$\pm$0.6 & 89.2$\pm$0.2 & \textbf{92.5}$\pm$0.9 & 88.3$\pm$1.2 & 21.3$\pm$0.7 \\
BoDA-Mr,c~\citep{Boda} & \textbf{90.4}$\pm$0.3 & 81.2$\pm$0.7 & 85.8$\pm$0.4 & \textbf{74.6}$\pm$0.7 & 83.0$\pm$0.4 & \textbf{74.6}$\pm$0.7 & 89.2$\pm$0.2 & 91.2$\pm$0.6 & \textbf{91.7}$\pm$2.0 & \textbf{21.7}$\pm$1.4 \\
\midrule
\textbf{Ours} & 89.9$\pm$0.5 & \textbf{84.7}$\pm$0.6 & 87.6$\pm$0.8 & 73.5$\pm$0.6 & \textbf{83.9}$\pm$0.2 & 73.5$\pm$0.6 & \textbf{91.4}$\pm$0.3 & \textbf{92.5}$\pm$0.6 & 88.3$\pm$3.4 & 21.4$\pm$1.2 \\
\bottomrule
\end{tabular}}
\end{table}

\renewcommand{\thefootnote}{\fnsymbol{footnote}}
\begin{table}[h]
\centering
\setlength{\tabcolsep}{4pt}
\caption{Accuracy (\%) under the multi-domain long-tailed (MDLT) setting on \textbf{DomainNet-MLT}. Results are sourced from their original papers, and are reported by domain and by shot.}
\label{tab:dnmlt_results}
\resizebox{\linewidth}{!}{%
\begin{tabular}{l|cccccccc|cccc}
\toprule
\multicolumn{1}{c}{}& \multicolumn{8}{c}{\textbf{Accuracy (by domain)}} & \multicolumn{4}{c}{\textbf{Accuracy (by shot)}} \\
\cmidrule(lr){2-9}\cmidrule(lr){10-13}
\multicolumn{1}{c}{\textbf{Algorithm}} & \textbf{Clip} & \textbf{Info} & \textbf{Paint} & \textbf{Quick} & \textbf{Real} & \textbf{Sketch} & \textbf{Average} & \multicolumn{1}{c}{\textbf{Worst}} & \textbf{Many} & \textbf{Medium} & \textbf{Few} & \textbf{Zero} \\
\midrule\midrule
ERM~\citep{ERM} & 68.6$\pm$0.1 & 29.4$\pm$0.3 & 57.1$\pm$0.2 & 62.8$\pm$0.3 & 72.1$\pm$0.2 & 61.7$\pm$0.2 & 58.6$\pm$0.2 & 29.4$\pm$0.3 & 66.0$\pm$0.1 & 56.1$\pm$0.1 & 35.9$\pm$0.5 & 27.6$\pm$0.3 \\
IRM~\citep{IRM} & 66.7$\pm$0.2 & 27.6$\pm$0.1 & 56.0$\pm$0.2 & 60.1$\pm$0.1 & 72.0$\pm$0.0 & 60.2$\pm$0.2 & 57.1$\pm$0.1 & 27.6$\pm$0.1 & 64.7$\pm$0.1 & 54.3$\pm$0.3 & 33.5$\pm$0.3 & 25.8$\pm$0.3 \\
GroupDRO~\citep{GroupDRO} & 60.1$\pm$0.2 & 25.9$\pm$0.2 & 50.3$\pm$0.1 & 63.9$\pm$0.2 & 64.9$\pm$0.2 & 56.7$\pm$0.3 & 53.6$\pm$0.1 & 25.9$\pm$0.2 & 61.8$\pm$0.1 & 49.1$\pm$0.3 & 30.7$\pm$0.7 & 22.0$\pm$0.1 \\
Mixup~\citep{Mixup} & 67.6$\pm$0.2 & 28.7$\pm$0.0 & 56.4$\pm$0.2 & 60.0$\pm$0.4 & 72.1$\pm$0.1 & 60.9$\pm$0.1 & 57.6$\pm$0.1 & 28.7$\pm$0.0 & 64.9$\pm$0.2 & 54.5$\pm$0.1 & 35.6$\pm$0.2 & 27.3$\pm$0.3 \\
MLDG~\citep{MLDG} & 68.0$\pm$0.2 & 28.7$\pm$0.1 & 57.2$\pm$0.1 & 61.6$\pm$0.2 & 73.3$\pm$0.1 & 61.9$\pm$0.2 & 58.5$\pm$0.0 & 28.7$\pm$0.1 & 66.0$\pm$0.1 & 55.7$\pm$0.1 & 35.3$\pm$0.2 & 26.9$\pm$0.3 \\
CORAL~\citep{DeepCORAL} & 69.1$\pm$0.3 & 30.1$\pm$0.4 & 57.8$\pm$0.2 & 63.4$\pm$0.2 & 72.8$\pm$0.2 & 63.3$\pm$0.3 & 59.4$\pm$0.1 & 30.1$\pm$0.4 & 66.4$\pm$0.1 & 57.1$\pm$0.0 & 37.7$\pm$0.6 & 29.9$\pm$0.2 \\
MMD~\citep{MMD} & 66.1$\pm$0.1 & 27.2$\pm$0.2 & 55.9$\pm$0.1 & 59.3$\pm$0.2 & 71.9$\pm$0.1 & 60.0$\pm$0.2 & 56.7$\pm$0.0 & 27.2$\pm$0.2 & 64.2$\pm$0.1 & 54.0$\pm$0.0 & 33.9$\pm$0.2 & 25.4$\pm$0.2 \\
DANN~\citep{DANN} & 65.5$\pm$0.3 & 26.9$\pm$0.4 & 55.2$\pm$0.1 & 57.4$\pm$0.2 & 70.6$\pm$0.1 & 59.0$\pm$0.2 & 55.8$\pm$0.1 & 26.9$\pm$0.4 & 63.0$\pm$0.1 & 52.7$\pm$0.1 & 34.2$\pm$0.4 & 26.8$\pm$0.4 \\
CDANN~\citep{CDANN} & 65.9$\pm$0.1 & 27.7$\pm$0.1 & 55.3$\pm$0.1 & 57.6$\pm$0.2 & 70.9$\pm$0.2 & 58.7$\pm$0.1 & 56.0$\pm$0.1 & 27.7$\pm$0.1 & 63.2$\pm$0.0 & 52.7$\pm$0.2 & 34.3$\pm$0.5 & 27.6$\pm$0.1 \\
MTL~\citep{blanchard2021mtl_marginal_transfer_learning} & 68.2$\pm$0.2 & 29.3$\pm$0.2 & 57.3$\pm$0.1 & 62.1$\pm$0.1 & 72.9$\pm$0.1 & 61.8$\pm$0.2 & 58.6$\pm$0.1 & 29.3$\pm$0.2 & 65.9$\pm$0.1 & 56.0$\pm$0.4 & 35.4$\pm$0.1 & 28.2$\pm$0.3 \\
SagNet~\citep{SAGNET} & 68.5$\pm$0.1 & 29.4$\pm$0.2 & 57.8$\pm$0.2 & 62.1$\pm$0.2 & 73.3$\pm$0.1 & 62.4$\pm$0.1 & 58.9$\pm$0.0 & 29.4$\pm$0.2 & 66.3$\pm$0.1 & 56.4$\pm$0.0 & 36.2$\pm$0.3 & 27.2$\pm$0.4 \\
Fish~\citep{Fish} & 68.7$\pm$0.1 & 29.1$\pm$0.1 & 58.4$\pm$0.1 & 64.1$\pm$0.1 & 73.9$\pm$0.1 & 63.7$\pm$0.1 & 59.6$\pm$0.1 & 29.1$\pm$0.1 & 67.1$\pm$0.1 & 57.2$\pm$0.1 & 36.8$\pm$0.4 & 27.8$\pm$0.3 \\
Focal~\citep{Lin2017Focal} & 67.6$\pm$0.1 & 27.5$\pm$0.1 & 56.5$\pm$0.3 & 62.3$\pm$0.3 & 71.7$\pm$0.3 & 61.4$\pm$0.3 & 57.8$\pm$0.2 & 27.5$\pm$0.1 & 65.2$\pm$0.2 & 55.1$\pm$0.2 & 35.8$\pm$0.1 & 26.3$\pm$0.1 \\
CBLoss~\citep{Cui2019CBLoss} & 68.3$\pm$0.2 & 30.1$\pm$0.1 & 57.8$\pm$0.1 & 60.8$\pm$0.1 & 73.3$\pm$0.2 & 63.3$\pm$0.1 & 58.9$\pm$0.1 & 30.1$\pm$0.1 & 64.3$\pm$0.0 & 61.0$\pm$0.3 & 42.5$\pm$0.4 & 28.1$\pm$0.2 \\
LDAM~\citep{Cao2019LDAM} & 68.8$\pm$0.2 & 29.2$\pm$0.2 & 57.1$\pm$0.1 & 65.0$\pm$0.0 & 72.3$\pm$0.1 & 63.1$\pm$0.1 & 59.2$\pm$0.0 & 29.2$\pm$0.2 & 66.6$\pm$0.0 & 57.0$\pm$0.0 & 37.1$\pm$0.2 & 27.8$\pm$0.3 \\
BSoftmax~\citep{Ren2020BalancedSoftmax} & 68.5$\pm$0.1 & 29.9$\pm$0.1 & 57.8$\pm$0.1 & 60.5$\pm$0.3 & 73.4$\pm$0.1 & 63.3$\pm$0.0 & 58.9$\pm$0.1 & 29.9$\pm$0.1 & 64.3$\pm$0.1 & 60.9$\pm$0.3 & 42.4$\pm$0.6 & 28.2$\pm$0.1 \\
SSP~\citep{SSP} & 69.7$\pm$0.1 & 31.6$\pm$0.2 & 58.8$\pm$0.1 & 59.7$\pm$0.3 & 73.9$\pm$0.1 & 64.2$\pm$0.1 & 59.7$\pm$0.0 & 31.6$\pm$0.2 & 64.3$\pm$0.1 & 62.6$\pm$0.1 & 45.0$\pm$0.3 & 30.5$\pm$0.0 \\
CRT~\citep{Kang2020cRT} & 70.0$\pm$0.1 & 31.6$\pm$0.1 & 59.2$\pm$0.2 & 64.0$\pm$0.1 & 73.4$\pm$0.1 & 64.4$\pm$0.1 & 60.4$\pm$0.2 & 31.6$\pm$0.1 & 66.8$\pm$0.0 & 61.6$\pm$0.1 & 45.7$\pm$0.1 & 29.7$\pm$0.1 \\
BoDAr~\citep{Boda} & 70.0$\pm$0.1 & 32.6$\pm$0.1 & 59.1$\pm$0.1 & 61.2$\pm$0.4 & 73.3$\pm$0.1 & 64.1$\pm$0.1 & 60.1$\pm$0.2 & 32.6$\pm$0.1 & 65.7$\pm$0.2 & 60.6$\pm$0.1 & 42.6$\pm$0.3 & 30.5$\pm$0.2 \\
BoDA-Mr~\citep{Boda} & 70.6$\pm$0.1 & 32.2$\pm$0.2 & 57.7$\pm$0.3 & 65.5$\pm$0.3 & 70.2$\pm$0.1 & 64.5$\pm$0.1 & 60.1$\pm$0.2 & 32.2$\pm$0.2 & 65.9$\pm$0.2 & 60.7$\pm$0.1 & 42.9$\pm$0.3 & 30.0$\pm$0.1 \\
BoDAr,c~\citep{Boda} & 72.0$\pm$0.2 & 33.4$\pm$0.1 & 60.7$\pm$0.2 & 63.6$\pm$0.2 & 74.6$\pm$0.1 & 65.5$\pm$0.2 & 61.7$\pm$0.1 & 33.4$\pm$0.1 & 67.0$\pm$0.1 & 62.7$\pm$0.1 & 46.0$\pm$0.2 & \textbf{32.2}$\pm$0.3 \\
BoDA-Mr,c~\citep{Boda} & 71.8$\pm$0.1 & 33.3$\pm$0.1 & 60.8$\pm$0.1 & 63.7$\pm$0.3 & 74.6$\pm$0.1 & 65.8$\pm$0.2 & 61.7$\pm$0.2 & 33.3$\pm$0.1 & 67.0$\pm$0.1 & \textbf{63.0}$\pm$0.3 & \textbf{46.6}$\pm$0.4 & 31.8$\pm$0.2 \\
\midrule
\textbf{Ours} & \textbf{72.8}$\pm$0.0 & \textbf{33.8}$\pm$0.1 & \textbf{63.4}$\pm$0.1 & \textbf{67.0}$\pm$0.2 & \textbf{78.9}$\pm$0.1 & \textbf{67.2}$\pm$0.0 & \textbf{63.8}$\pm$0.1 & \textbf{33.8}$\pm$0.1 & \textbf{71.0}$\pm$0.1 & 62.3$\pm$0.1 & 41.9$\pm$0.2 & 31.6$\pm$0.3 \\
\bottomrule
\end{tabular}}
\end{table}

\renewcommand{\thefootnote}{\fnsymbol{footnote}}
\begin{table}[h]
\centering
\setlength{\tabcolsep}{4pt}
\caption{Accuracy (\%) under the multi-domain long-tailed (MDLT) setting on \textbf{DomainNet126-MLT}. Results are sourced from their original papers, and are reported by domain and by shot.}
\label{tab:dn126mlt_results}
\resizebox{\linewidth}{!}{%
\begin{tabular}{l|cccccc|cccc}
\toprule
\multicolumn{1}{c}{}& \multicolumn{6}{c}{\textbf{Accuracy (by domain)}} & \multicolumn{4}{c}{\textbf{Accuracy (by shot)}} \\
\cmidrule(lr){2-7}\cmidrule(lr){8-11}
\multicolumn{1}{c}{\textbf{Algorithm}} & \textbf{C} & \textbf{P} & \textbf{R} & \textbf{S} & \textbf{Average} & \multicolumn{1}{c}{\textbf{Worst}} & \textbf{Many} & \textbf{Medium} & \textbf{Few} & \textbf{Zero} \\
\midrule\midrule
ERM~\citep{ERM}       & 78.9$\pm$0.4 & 71.1$\pm$0.3 & 85.6$\pm$0.1 & 78.6$\pm$0.3 & 78.6$\pm$0.1 & 71.1$\pm$0.3 & 84.1$\pm$0.1 & 78.5$\pm$0.2 & 62.7$\pm$0.6 & 48.3$\pm$0.5 \\
IRM~\citep{IRM}       & 78.4$\pm$0.0 & 72.9$\pm$0.0 & 86.5$\pm$0.0 & 78.7$\pm$0.0 & 79.1$\pm$0.0 & 72.9$\pm$0.0 & 84.6$\pm$0.0 & 79.4$\pm$0.0 & 63.0$\pm$0.0 & 47.0$\pm$0.0 \\
GroupDRO~\citep{GroupDRO}  & 77.4$\pm$0.3 & 70.0$\pm$0.5 & 86.6$\pm$0.1 & 77.8$\pm$0.2 & 77.9$\pm$0.3 & 70.0$\pm$0.5 & 83.9$\pm$0.2 & 77.5$\pm$0.2 & 60.5$\pm$1.7 & 50.0$\pm$1.1 \\
Mixup~\citep{Mixup}     & 80.6$\pm$0.2 & 73.1$\pm$0.3 & 86.9$\pm$0.1 & 80.2$\pm$0.2 & 80.2$\pm$0.1 & 73.1$\pm$0.3 & 85.3$\pm$0.0 & 80.6$\pm$0.1 & 65.6$\pm$1.2 & 49.5$\pm$0.1 \\
MLDG~\citep{MLDG}      & 79.0$\pm$0.1 & 70.9$\pm$0.2 & 85.9$\pm$0.1 & 78.2$\pm$0.1 & 78.5$\pm$0.0 & 70.9$\pm$0.2 & 84.1$\pm$0.1 & 78.7$\pm$0.2 & 62.4$\pm$0.3 & 47.8$\pm$0.6 \\
CORAL~\citep{DeepCORAL}     & 79.4$\pm$0.2 & 73.0$\pm$0.1 & 86.6$\pm$0.2 & 79.3$\pm$0.1 & 79.6$\pm$0.1 & 73.0$\pm$0.1 & 84.9$\pm$0.1 & 79.6$\pm$0.1 & 64.0$\pm$0.2 & 51.7$\pm$0.9 \\
MMD~\citep{MMD}       & 79.2$\pm$0.0 & 72.3$\pm$0.0 & 86.4$\pm$0.0 & 79.5$\pm$0.0 & 79.3$\pm$0.0 & 72.3$\pm$0.0 & 84.5$\pm$0.0 & 79.8$\pm$0.0 & 63.9$\pm$0.0 & 49.5$\pm$0.0 \\
DANN~\citep{DANN}      & 74.7$\pm$0.5 & 68.8$\pm$0.2 & 82.1$\pm$0.3 & 75.2$\pm$0.1 & 75.2$\pm$0.2 & 68.8$\pm$0.2 & 80.7$\pm$0.3 & 74.5$\pm$0.3 & 60.5$\pm$0.8 & 47.7$\pm$0.9 \\
CDANN~\citep{CDANN}     & 75.8$\pm$0.5 & 69.1$\pm$0.3 & 82.0$\pm$0.3 & 75.4$\pm$0.3 & 75.6$\pm$0.2 & 69.1$\pm$0.3 & 80.6$\pm$0.2 & 75.5$\pm$0.2 & 61.3$\pm$0.7 & 48.0$\pm$0.7 \\
MTL~\citep{blanchard2021mtl_marginal_transfer_learning} & 78.1$\pm$0.1 & 70.6$\pm$0.4 & 85.9$\pm$0.2 & 77.8$\pm$0.2 & 78.1$\pm$0.2 & 70.6$\pm$0.4 & 83.9$\pm$0.1 & 78.2$\pm$0.4 & 61.4$\pm$0.7 & 45.5$\pm$0.6 \\
SagNet~\citep{SAGNET}  & 79.0$\pm$0.1 & 71.5$\pm$0.4 & 85.5$\pm$0.3 & 79.0$\pm$0.2 & 78.8$\pm$0.2 & 71.5$\pm$0.4 & 84.1$\pm$0.2 & 79.2$\pm$0.3 & 61.3$\pm$0.8 & 49.4$\pm$1.0 \\
Fish~\citep{Fish}      & 79.0$\pm$0.4 & 70.7$\pm$0.3 & 86.2$\pm$0.1 & 78.5$\pm$0.2 & 78.6$\pm$0.1 & 70.7$\pm$0.3 & 84.2$\pm$0.2 & 78.6$\pm$0.4 & 63.0$\pm$0.3 & 47.6$\pm$1.0 \\
Focal~\citep{Lin2017Focal}  & 78.1$\pm$0.3 & 70.0$\pm$0.2 & 85.1$\pm$0.3 & 78.1$\pm$0.2 & 77.8$\pm$0.2 & 70.0$\pm$0.2 & 83.6$\pm$0.1 & 77.8$\pm$0.2 & 61.8$\pm$0.5 & 45.8$\pm$0.4 \\
CBLoss~\citep{Cui2019CBLoss} & 79.0$\pm$0.2 & 71.3$\pm$0.4 & 85.8$\pm$0.3 & 79.0$\pm$0.1 & 78.8$\pm$0.2 & 71.3$\pm$0.4 & 83.3$\pm$0.3 & 80.0$\pm$0.3 & 65.2$\pm$0.1 & 48.8$\pm$0.8 \\
LDAM~\citep{Cao2019LDAM}    & 79.0$\pm$0.1 & 71.1$\pm$0.1 & 86.2$\pm$0.2 & 78.2$\pm$0.3 & 78.6$\pm$0.1 & 71.1$\pm$0.1 & 83.8$\pm$0.1 & 79.2$\pm$0.1 & 63.2$\pm$1.1 & 49.2$\pm$1.1 \\
BSoftmax~\citep{Ren2020BalancedSoftmax} & 79.9$\pm$0.3 & 71.9$\pm$0.4 & 86.0$\pm$0.1 & 79.9$\pm$0.0 & 79.4$\pm$0.2 & 71.9$\pm$0.4 & 83.2$\pm$0.0 & 80.4$\pm$0.1 & \textbf{68.1}$\pm$0.9 & \textbf{55.0}$\pm$0.8 \\
CRT~\citep{Kang2020cRT}     & 80.2$\pm$0.2 & 72.5$\pm$0.1 & 87.2$\pm$0.1 & 80.8$\pm$0.1 & 80.2$\pm$0.0 & 72.5$\pm$0.1 & 85.3$\pm$0.0 & 80.7$\pm$0.1 & 65.1$\pm$0.4 & 51.0$\pm$0.5 \\
BoDA~\citep{Boda}     & 80.0$\pm$0.2 & 73.1$\pm$0.1 & 86.4$\pm$0.2 & 79.4$\pm$0.2 & 79.7$\pm$0.0 & 73.1$\pm$0.1 & 84.5$\pm$0.1 & 80.4$\pm$0.1 & 65.8$\pm$0.7 & 51.7$\pm$1.2 \\
\midrule
\textbf{Ours} & \textbf{81.0}$\pm$0.2 & \textbf{74.2}$\pm$0.2 & \textbf{88.4}$\pm$0.2 & \textbf{80.9}$\pm$0.1 & \textbf{81.1}$\pm$0.2 & \textbf{74.2}$\pm$0.2 & \textbf{86.5}$\pm$0.1 & \textbf{81.6}$\pm$0.2 & 64.6$\pm$0.4 & 51.0$\pm$0.4 \\
\bottomrule
\end{tabular}}
\end{table}





\clearpage

\section{Limitations and Future Research}

While this work introduces a novel theoretical framework and the \textsc{RC-Align} algorithm for domain generalization under compound distribution shifts, several limitations exist that suggest directions for future research.

\paragraph{Architectural Dependence and Scalability.}
Our experimental evaluation was conducted using a ResNet-50 backbone pre-trained on ImageNet, following standard protocols in the domain generalization literature. While this ensures fair comparison with existing work, it remains unverified whether the benefits of \textsc{RC-Align} are specific to CNN architectures or if they extend to more recent architectures like Vision Transformers (ViTs). Furthermore, its effectiveness in settings where models are trained from scratch has not been explored. Future work should validate the generalizability of the proposed methodology across a diverse range of backbone architectures.

\paragraph{Computational Overhead of Meta-Learning.}
\textsc{RC-Align} employs a MAML-style meta-learning framework involving a nested optimization process. Although the use of the First-Order MAML (FO-MAML) approximation mitigates the cost to a reasonable trade-off (approximately $1.45\times$ that of ERM), the method inevitably demands higher computational resources than standard training. This additional computational burden remains a constraint that should be considered, particularly when scaling to massive datasets or operating in resource-limited environments.

\paragraph{Centroid Estimation in Long-Tail Scenarios.}
The proposed Domain-Class Distribution Alignment (DA) loss relies on the accurate estimation of class-wise centroids from the source domains. In this study, centroids are estimated on-the-fly using mini-batches for computational efficiency within the meta-learning loop. However, this approach can introduce noise and instability in the estimates, particularly for tail classes in the Multi-Domain Long-Tail (MDLT) setting where data is sparse. Inaccurate centroid estimation could undermine the effectiveness of feature alignment, necessitating the exploration of more robust centroid estimation techniques that do not significantly increase computational complexity.

\paragraph{Assumptions in Theoretical Analysis.}
The theoretical analysis of the meta-learning procedure's optimization dynamics (Section~\ref{sec:maml_analysis}) relies on strong assumptions such as L-smoothness and the Polyak--Lojasiewicz (PL) condition to characterize convergence behavior. As acknowledged in the paper, the PL condition may not hold globally for the highly non-convex loss landscapes of deep neural networks. Therefore, the theoretical guarantees regarding optimization should be interpreted as holding under local or idealized settings.

\paragraph{Theoretical Scope on Domain Interpolation.}
Consistent with prevalent assumptions in Domain Generalization literature~\cite{muandet2013domain, blanchard2021mtl_marginal_transfer_learning, wang2022generalizing, rosenfeld2022online}, our theoretical framework primarily addresses scenarios where target domains lie within the support of source domains (i.e., interpolation). This focus reflects a standard characteristic of mixture-based DG analysis. Addressing generalization to extrapolated targets often necessitates modeling causal mechanisms or mitigating negative transfer effects~\cite{IRM, mahajan2021domain, sheth2022domain, lv2022causality}, which presents a challenge orthogonal to our current problem setting. We identify robustness against extrapolation as a promising direction for future research.

\clearpage





\end{document}